\documentclass[twoside]{article}

\usepackage[accepted]{aistats2025}
\usepackage[round]{natbib}

\usepackage{microtype}
\usepackage{graphicx}
\usepackage{subfigure}
\usepackage{booktabs} % for professional tables

\usepackage{amsmath}
\usepackage{amssymb}
\usepackage{mathtools}
\usepackage{amsthm}
\usepackage{float}
\usepackage{hyperref}

\theoremstyle{plain}
\newtheorem{theorem}{Theorem}[section]

\newtheorem{lemma}[theorem]{Lemma}

\theoremstyle{definition}
\newtheorem{definition}[theorem]{Definition}

\theoremstyle{remark}

\begin{document}

% If your paper is accepted and the title of your paper is very long,
% the style will print as headings an error message. Use the following
% command to supply a shorter title of your paper so that it can be
% used as headings.
%
%\runningtitle{I use this title instead because the last one was very long}

% If your paper is accepted and the number of authors is large, the
% style will print as headings an error message. Use the following
% command to supply a shorter version of the authors names so that
% they can be used as headings (for example, use only the surnames)
%
%\runningauthor{Surname 1, Surname 2, Surname 3, ...., Surname n}

\twocolumn[

\aistatstitle{Rethinking Neural-based Matrix Inversion: Why can't, and Where can}

\aistatsauthor{ Yuliang Ji \And Jian Wu \And  Yuanzhe Xi }

\aistatsaddress{ Nanjing University of \\ Science and Technology \\ yuliang.ji@njust.edu.cn \And  Tokyo Institute of Technology \\ wu.j.as@m.titech.ac.jp \And Emory University \\ yuanzhe.xi@emory.edu} 
]

% \aistatsaddress{ Nanjing University of  Science and Technology \And  Tokyo Institute of Technology \And Emory University } ]

% \twocolumn[

% \aistatstitle{Instructions for Paper Submissions to AISTATS 2025}

% \aistatsauthor{ Author 1 \And Author 2 \And  Author 3 }

% \aistatsaddress{ Institution 1 \And  Institution 2 \And Institution 3 } ]

\begin{abstract}
  Deep neural networks have achieved substantial success across various scientific computing tasks. A pivotal challenge within this domain is the rapid and parallel approximation of matrix inverses, critical for numerous applications. Despite significant progress, there currently exists no universal neural-based method for approximating matrix inversion. This paper presents a theoretical analysis demonstrating the fundamental limitations of neural networks in developing a general matrix inversion model. We expand the class of Lipschitz functions to encompass a wider array of neural network models, thereby refining our theoretical approach. Moreover, we delineate specific conditions under which neural networks can effectively approximate matrix inverses. Our theoretical results are supported by experimental results from diverse matrix datasets, exploring the efficacy of neural networks in addressing the matrix inversion challenge.
\end{abstract}

\section{Introduction}

In recent years, neural network-based methods have significantly advanced the solution of complex problems in scientific computing. Notably, deep neural networks have been effectively applied to eigenvalue problems for linear and semilinear second-order differential operators in high dimensions \citep{HAN2020109792}. Additionally, neural networks have introduced novel approaches to solving eigenvalue problems for differential self-adjoint operators \citep{10.1162/neco_a_01583}. Among the most striking developments, researchers have employed reinforcement learning in conjunction with neural networks to develop several rapid matrix multiplication algorithms \citep{matrixmultiplication22Nature}, highlighting the expanding capabilities of neural technologies in computational methodologies.

A fundamental challenge in scientific computing is the fast and parallel approximation of matrix inverses. This issue has long attracted substantial research interest, leading to the development of several classical methods. Techniques such as LU decomposition, Cholesky decomposition, QR decomposition, and the Gauss-Jordan method are well-established for stable matrix inversion \citep{golub1996matrix}. However, these methods are primarily sequential algorithms, which can restrict their performance within parallel computing frameworks \citep{dongarra1990solving}. For example, the inherent sequential steps in LU decomposition limit its efficiency on parallel architectures. In contrast, neural networks present a promising alternative, harnessing the capabilities of modern computing architectures to develop innovative solutions for quickly and effectively approximating matrix inverses.

The concept of a general matrix inversion model encompasses a method capable of approximating the matrix inversion operation across a broad spectrum of the space $\mathbb{R}^{n\times n}$ with minimal error. Despite some claims that neural networks can accurately perform matrix inversion tasks, a comprehensive, end-to-end neural network model for general matrix inversion remains elusive. Previous studies \citep{NIPS1987_6c8349cc, 82neural_matrix_inv82luo, Steriti90matrix_inv} have introduced neural-based methods for matrix inversion; however, these methods are typically confined to specific training domains. Their performance deteriorates significantly when applied outside these domains or when the models encounter particular types of matrices. Moreover, attempts to integrate neural networks with Newton iteration \citep{newtoniter09zhang} have also been restricted to narrow operational scopes. Other methods that combine neural networks with optimization techniques \citep{https://doi.org/10.1049/sil2.12156, 10019100} often struggle with convergence issues on certain datasets. Although recent research \citep{10115024, Dimitrios2023, 9468339} has demonstrated that neural-based models excel in handling time-varying matrix inversion challenges, these scenarios are distinct from traditional matrix inversion as they incorporate temporal variables and typically rely on a known initial solution.

In this paper, we investigate the existence of a neural network model capable of solving the general matrix inversion problem. Our primary research question is: \emph{Can a neural network, trained in an end-to-end fashion, accurately approximate the matrix inverse across the entire space of $\mathbb{R}^{n \times n}$ under mild assumptions?} Previous theoretical research \citep{lipschitz21Kim, lipschitz18virmaux, lipschitz20Larorre} often relies on Lipschitz continuity to interpret the capacity of neural networks. However, several modern neural network architectures, such as those  employing residual connections, are not Lipschitz continuous \citep{anil2019sorting}. To conduct our analysis, we introduce a generalization of the Lipschitz function class, which we refer to as the \emph{polynomial Lipschitz continuity}. This class encompasses a broader range of functions. By leveraging this generalization, we can more accurately characterize the capabilities and limitations of neural networks in approximating matrix inverses.

To substantiate our theoretical arguments, we select several specific datasets for training end-to-end neural networks. After training, we perform both experimental and theoretical analyses to understand what the models have learned and how they perform within the space of the selected datasets. Our methodological approach combines theoretical proofs with empirical validation, providing a comprehensive examination of neural networks' ability to approximate matrix inverses.

The main contributions of this paper are as follows:

\begin{itemize}
    \item We introduce the \emph{polynomial Lipschitz continuity} that mathematically characterizes a wide range of neural network architectures, including those not covered by traditional Lipschitz continuity. This generalization provides a more universal property of modern neural networks, facilitating deeper theoretical analysis.
    \item We provide the first proofs explaining the absence of neural network-based end-to-end matrix inversion models that achieve satisfactory accuracy across the entire space of $\mathbb{R}^{n \times n}$ under mild assumptions. Our results highlight fundamental limitations in the capacity of neural networks to generalize in this context.
    \item We identify specific regions within $\mathbb{R}^{n \times n}$ where neural networks can accurately approximate matrix inverses. Through both experimental and theoretical analyses, we elucidate what neural networks learn when solving matrix inversion problems, shedding light on their practical applicability and limitations.
\end{itemize}

%%%%%%%%%%%%%%%%%%%%%%%%%%%%%%%%%%%%%%%%%%%%%%%%%%%%%%%%%%%%
\section{Lipschitz Continuity and Its Generalization}

In this section, we introduce the notation used throughout the paper and discuss the concept of Lipschitz continuity, emphasizing the necessity of generalizing it to encompass modern neural network architectures.

Let \( \text{Inv}(x) \) denote the matrix inversion function, where \( x \) is an input matrix. We denote the trained neural network model as \( F(x) \), which takes a matrix as input and outputs an approximation of its inverse. We use \( \|\cdot\|_L \) to represent a general norm, which could be any of the \( L_1 \), \( L_2 \), or \( L_{\infty} \) norms, as the shared properties of these norms are pertinent to our proofs. To avoid ambiguity, we define the \( L\) norm of a matrix to be a vectorized norm. For example, the \( L_2 \) norm of matrix $A$ is defined as $||A||_{L_2}=\sqrt{\sum\limits_{i,j} |a_{ij}|^2}$.

Assuming the dimension of the input matrices is fixed at \( n \times n \), it is well-known that the set of singular matrices has zero measure in \( \mathbb{R}^{n \times n} \) under the Lebesgue measure. Thus, without loss of generality, we assume that the dataset \( M \) contains no singular matrices and has positive measure in \( \mathbb{R}^{n \times n} \). 

We say that the well-trained neural network model \( F(x) \) can approximate the target function \( \text{Inv}(x) \) if the following inequality holds:
\begin{equation}\label{main_evaluation_formula}
    \mathbb{E}_{x \sim M}\left[\|\text{Inv}(x) - F(x)\|_L^k\right] < \epsilon,
\end{equation}
where \( \mathbb{E}_{x \sim M} \) denotes the expected value over the dataset \( M \), \( k \) is a positive integer, and \( \epsilon \) is a predefined small positive constant representing the acceptable error.

For example, choosing the \( L_2 \) norm with \( k=2 \) makes the above equation equivalent to the mean squared error between the output of the trained neural network model and the target, with the error restricted to be smaller than \( \epsilon \).

Other forms of evaluation, such as the \( k \)-th moment of relative error \( \mathbb{E}_{x \sim M}\left[\frac{\|\text{Inv}(x) - F(x)\|_L^K}{\|x\|_L^{K'}}\right] \), will be discussed in Appendix \ref{Appendix_proofs}. Our proof techniques are adaptable to alternative evaluation metrics for specific tasks.

\subsection{Lipschitz Continuity}

Lipschitz continuity is a fundamental concept in analysis and plays a crucial role in understanding the behavior of functions, especially in the context of approximation and generalization.

\begin{definition}[Lipschitz Continuity]
    Given two metric spaces $(X, d_X)$ and $(Y, d_Y)$, a function $f: X \rightarrow Y$ is called \emph{Lipschitz continuous} (or $K$-Lipschitz) if there exists a constant $K \geq 0$ such that
    \begin{equation}
        d_Y(f(x_1), f(x_2)) \leq K d_X(x_1, x_2), \quad \forall x_1, x_2 \in X.
    \end{equation}
\end{definition}

However, many neural network architectures do not conform to Lipschitz continuity. For instance, previous research \citep{lipschitz21Kim} has demonstrated that multi-head dot-product attention cannot be Lipschitz continuous. These observations necessitate generalizing the traditional notion of Lipschitz continuity to accommodate modern neural networks.

\subsection{Polynomial Lipschitz Continuity}

To effectively articulate our proofs, we extend the concept of Lipschitz continuity to what we term \emph{polynomial Lipschitz continuity}, applicable within the space $\mathbb{R}^n$ with standard norms.

\begin{definition}\label{definition_2_2}
A function $f(x):\mathbb{R}^{n_1}\rightarrow \mathbb{R}^{n_2}$ is called a polynomial Lipschitz continuous function under two norms $L^{+}, L^{*}$ defined on $\mathbb{R}^{n_1}, \mathbb{R}^{n_2}$ if it satisfies
\begin{equation}
    \|f(x)-f(y)\|_{L^*}\leq \sum\limits_{i=0}^{n^f} f_i(\|x\|_{L^{+}}, \|y\|_{L^{+}}) \|x-y\|_{L^{+}}^i,
\end{equation}
for any $x,y \in \mathbb{R}^{n_1}$, where $f_i(\|x\|_{L^{+}}, \|y\|_{L^{+}})$ is a polynomial with variables $\|x\|_{L^{+}}, \|y\|_{L^{+}}$ and $n^f$ is constant, depending on the function $f(x)$.
\end{definition} 

This generalization is necessary because modern neural networks often involve components like activation functions and attention mechanisms that do not satisfy Lipschitz conditions but still exhibit controlled growth, allowing for meaningful analysis. This definition also generalizes Hölder's continuity.

\section{Why can't: Limitations of Neural Networks in Matrix Inversion}

In this section, we present our theoretical analysis of the limitations of neural networks in approximating matrix inverses. We provide brief summaries of the proofs, with full details available in Appendix \ref{Appendix_proofs}. This section is organized as follows:

\begin{itemize}
    \item In Subsection \ref{sec:pointwise}, we analyze the pointwise approximation capabilities of modern neural network models as the matrix inversion function.
    \item In Subsection \ref{sec:subset_average}, we discuss the theoretical analysis of the expected approximation error over a subset of the dataset $M$.
    \item In Subsection \ref{sec:general_expectation}, we extend our analysis to the expected approximation error over the dataset $M$.
\end{itemize}

\subsection{Pointwise Approximation}\label{sec:pointwise}

First, we analyze the pointwise performance of modern neural network models in approximating matrix inverses. We begin by proving that the composition of two polynomial Lipschitz continuous functions is also polynomial Lipschitz continuous.

\begin{lemma}\label{lemma_composition_functions}
    Let $f: \mathbb{R}^{n_2} \rightarrow \mathbb{R}^{n_3}$ and $g: \mathbb{R}^{n_1} \rightarrow \mathbb{R}^{n_2}$ be polynomial Lipschitz continuous functions under certain norms (either $L_1$, $L_2$, or $L_\infty$) defined on $\mathbb{R}^{n_i}$. Then, the composition $h = f \circ g: \mathbb{R}^{n_1} \rightarrow \mathbb{R}^{n_3}$ is also a polynomial Lipschitz continuous function.
\end{lemma}

The proof can be found in Appendix \ref{Appendix_A_1}.

We also establish that the combination (concatenation) of polynomial Lipschitz continuous functions is polynomial Lipschitz continuous.

\begin{lemma}\label{lemma_parallel_functions}
    Let $f: \mathbb{R}^{n_1} \rightarrow \mathbb{R}^{n_2}$ and $g: \mathbb{R}^{n_1} \rightarrow \mathbb{R}^{n_3}$ be polynomial Lipschitz continuous functions under certain norms. Then, the function $h(x) = (f(x), g(x)): \mathbb{R}^{n_1} \rightarrow \mathbb{R}^{n_2 + n_3}$ is also polynomial Lipschitz continuous.
\end{lemma}

The proof is provided in Appendix \ref{Appendix_A_2}.

Next, we establish a connection between the elements in the Jacobian of a function and its polynomial Lipschitz continuity to show that certain modern neural network structures possess this property.

\begin{lemma}\label{lemma_lipschitz_Jacobian}
    Let $f: \mathbb{R}^{n_1} \rightarrow \mathbb{R}^{n_2}$ be a function whose Jacobian exists everywhere. If each element of the Jacobian is bounded by a polynomial in $\|x\|_L$, then $f$ is polynomial Lipschitz continuous.
\end{lemma}

The full proof is in Appendix \ref{Appendix_A_Jacobian}.

We also highlight a significant property of Lipschitz continuous functions in the next lemma.

\begin{lemma}\label{poly_lip_func_bounded}
    Let $f$ be a Lipschitz continuous function defined on a bounded set $M \subset \mathbb{R}^{n}$. Then, $\|f(x)\|_L$ is bounded on the set $M$.
\end{lemma}

\begin{proof}
    Select a point $x_0$ in set $M$. In Definition \ref{definition_2_2}, consider $\|f(x)-f(x_0)\|_{L}$, it is easy to see $\|f(x)\|_L$ is bounded.
\end{proof}

We then demonstrate that many neural network structures are polynomial Lipschitz continuous.

\begin{lemma}\label{lemma_lipschitz_neural_network}
    Many modern neural network architectures are polynomial Lipschitz continuous functions.
\end{lemma}

\begin{proof}
    We consider several widely used neural network components and show that they are polynomial Lipschitz continuous:

    \begin{itemize}
        \item \textbf{Fully Connected Layers, Convolutional Layers, Activation Functions (ReLU, sigmoid, tanh):} These components are Lipschitz continuous under standard norms \citep{lipschitz21Kim}, and thus are polynomial Lipschitz continuous.
        
        \item \textbf{Neural Spline Layers:} As introduced in \citep{neuralspline19, autm22cai}, these layers involve element-wise polynomial functions, such as quadratic and cubic terms. Since the derivatives of these polynomials are bounded by polynomials in $\|x\|_L$, Lemma \ref{lemma_lipschitz_Jacobian} implies they are polynomial Lipschitz continuous.

        \item \textbf{Attention Layers:} Following \citep{lipschitz21Kim}, the elements of Jacobian of multi-head dot-product attention layers are bounded by polynomials in $\|X\|_L$, where $X$ is the input. Thus, they are polynomial Lipschitz continuous by Lemma \ref{lemma_lipschitz_Jacobian}.

        \item \textbf{Transformer Layers:} Transformers combine matrix multiplication, residual connections, multi-head attention, and activation functions.  As a result, their polynomial Lipschitz property depends on the polynomial Lipschitz continuity of the selected activation functions.
    \end{itemize}

    As there are many widely used structures, we do not list them all here. Additional examples and proofs are provided in Appendix \ref{Appendix_A_3}. 
\end{proof}

If the well-trained model is the composition of polynomial Lipschitz continuous neural-based blocks,  Lemma \ref{lemma_composition_functions} implies that the well-trained model is a polynomial Lipschitz continuous function.

Having established the polynomial Lipschitz continuity of these neural network structures, we focus on the behavior of the matrix inversion function near singular matrices.

\begin{lemma}\label{lemma_3.6}
    Let $A_0 \in \mathbb{R}^{n \times n}$ be a singular matrix of rank $n-1$, and let $B(A_0, \delta)$ denote the ball centered at $A_0$ with radius $\delta$ in $\mathbb{R}^{n \times n}$. Denote $S_B$ as the set of all singular matrices in $B(A_0, \delta)$. Then, there exists a $\delta > 0$ such that for any matrix $A \in B(A_0, \delta) \setminus S_B$, we have
    \[
    \|\text{Inv}(A)\|_L > \frac{C_{A_0}}{\|A - A_0\|_L},
    \]
    where $C_{A_0}$ is a positive constant depending on $A_0$.
\end{lemma}

The proof is provided in Appendix \ref{Appendix_lemma_3.6_proof}.

Based on the previous lemmas, we present our main theorem analyzing the pointwise approximation capability of polynomial Lipschitz continuous functions for matrix inversion.

\begin{theorem}\label{main_theorem_pointwise}
    \textbf{(Pointwise Approximation)} Let $M \subset \mathbb{R}^{n \times n}$ be a dataset, and let $B(\vec{a}, c) \subset M$ be a ball of sufficiently large radius $c$ centered at some point $\vec{a}$. Exclude all singular matrices from $M$. Then, under any norm $L$, for any polynomial Lipschitz continuous function $F(x)$ and any error threshold $E > 0$, there exists a data point $x \in M$ such that
    \[
    \|\text{Inv}(x) - F(x)\|_L > E.
    \]
\end{theorem}

\begin{proof}
    Since $B(\vec{a}, c)$ is a ball of sufficiently large radius, it must contain a singular matrix $A_0$ of rank $n-1$. From Lemma \ref{lemma_3.6}, in any neighborhood of $A_0$, $\|\text{Inv}(A)\|_L$ becomes unbounded as $A$ approaches $A_0$. Meanwhile, from Lemma \ref{poly_lip_func_bounded}, the polynomial Lipschitz continuous function $F(x)$ is bounded on $B(A_0, \delta)$. Therefore, we can choose $x$ sufficiently close to $A_0$ such that
    \[
    \|\text{Inv}(x)\|_L > E + \sup_{x \in B(A_0, \delta)} \|F(x)\|_L,
    \]
    implying that $\|\text{Inv}(x) - F(x)\|_L > E$.
\end{proof}

Theorem \ref{main_theorem_pointwise} reveals that any polynomial Lipschitz continuous function, including modern neural network models, cannot approximate the matrix inversion function pointwisely over general datasets.

\subsection{Expectation over a Subset}\label{sec:subset_average}

Beyond pointwise errors, we analyze the expectation of the approximation error over subsets of the dataset, providing insights into the average performance of neural networks in this context.

\begin{theorem}\label{main_theorem_mean}
    \textbf{(Subset Expectation)} Let $M \subset \mathbb{R}^{n \times n}$ be a dataset containing a ball $B(\vec{a}, c)$ of sufficiently large radius $c$ centered at some point $\vec{a}$. Exclude the set of all singular matrices from $M$, which is denoted as $S_M$. Under any norm $L$, for any polynomial Lipschitz continuous function $F(x)$ and any positive real number $E > 0$, there exists a subset $M_{\epsilon} \subset M\setminus S_M$ with positive measure such that
    \[
    \mathbb{E}_{x \sim M_{\epsilon}}\left[\|\text{Inv}(x) - F(x)\|_L^k\right] > E,
    \]
    for any $k > 0$.
\end{theorem}

\begin{proof}
Since $B(\vec{a}, c)$ is contained in $M$ and $c$ is sufficiently large, there should be a singular matrix $A_0$ with rank $n-1$, and a ball $B(A_0, \delta_0)\setminus S_M$  contained in $M$.

From Lemma \ref{lemma_3.6}, we know there exists a $\delta$ which satisfies that for any $A$ in $B(A_0, \delta) \setminus S_M$, 
    \[
    \|\text{Inv}(A)\|_L > \frac{C_{A_0}}{\|A - A_0\|_L}.
    \]
Let $C_F = \sup_{x \in B(A_0, \delta)} \|F(x)\|_L$, which is finite due to Lemma \ref{poly_lip_func_bounded}. We then define the subset $M_{\epsilon} = B(A_0, \epsilon) \subset M$, where $\epsilon$ is chosen such that $\epsilon \leq \min\left(\delta_0,\delta, \frac{C_{A_0}}{2 n C_F}\right)$. Then for any $x \in M_{\epsilon}$,
    \[
    \|\text{Inv}(x) - F(x)\|_L > \frac{C_{A_0}}{\|x - A_0\|_L} - C_F \geq \frac{C_{A_0}}{2 \|x - A_0\|_L}.
    \]

    Thus, the expected error over $M_{\epsilon}$ is then lower bounded by
    \[
    \mathbb{E}_{x \sim M_{\epsilon}}\left[\|\text{Inv}(x) - F(x)\|_L^k\right] \geq \frac{C}{\epsilon^{n^2}} \int_0^{\epsilon} r^{n^2 - 1 - k} dr,
    \]
    where $C$ is a constant depending on $C_{A_0}$ and $n$. For $k > 0$, the integral diverges as $\epsilon \to 0$, meaning we can make the expected error exceed any $E > 0$ by choosing $\epsilon$ small enough.
\end{proof}

Theorem \ref{main_theorem_mean} indicates that even over subsets of the dataset, polynomial Lipschitz continuous neural network models cannot achieve arbitrarily small expected errors in approximating matrix inverses.

\subsection{Expectation over a General Set}\label{sec:general_expectation}

Finally, we analyze the expectation of the approximation error over the entire dataset $M$.

\begin{theorem}\label{main_theorem_general}
    \textbf{(General Expectation)} Under the same assumptions as Theorem \ref{main_theorem_mean}, for any polynomial Lipschitz continuous function $F(x)$ and any $k > n^2$, the expected error over $M$ is infinite:
    \[
    \mathbb{E}_{x \sim M}\left[\|\text{Inv}(x) - F(x)\|_L^k\right] = +\infty.
    \]
\end{theorem}

\begin{proof}
    Using the same estimation as in Theorem \ref{main_theorem_mean}, we have
\begin{equation}
\begin{split}
     &\mathbb{E}_{x \sim M}\left[\|\text{Inv}(x) - F(x)\|_L^k\right]\\
    \geq  & \int_{M_{\epsilon}} \|\text{Inv}(x) - F(x)\|_L^k \, \frac{dm}{m(M)} \geq C\int^{\epsilon}_{0}r^{n^2-1-k} dr,
\end{split}
\end{equation}
where $C$ represents a real number calculated from $n,C_{A_0},m(M)$. The integral over $M_{\epsilon}$ diverges for $k > n^2$ and $\epsilon\rightarrow 0$, leading to an infinite expected error over $M$.
\end{proof}

Theorem \ref{main_theorem_general} demonstrates that over the entire dataset, polynomial Lipschitz continuous neural networks cannot achieve finite expected errors for large $k$.  

Based on Theorems \ref{main_theorem_pointwise}, \ref{main_theorem_mean}, and \ref{main_theorem_general}, we conclude that neural networks modeled as polynomial Lipschitz continuous functions struggle to approximate the matrix inversion operation over general spaces. This is due to the unbounded behavior of the matrix inversion function near singular matrices, which cannot be captured by the bounded nature of polynomial Lipschitz continuous functions.

\section{Where can: Feasible Regions for Neural Network Approximation}

In this section, we discuss the regions in which it is possible to train an end-to-end neural network model to approximate the matrix inversion function. We also describe how to design an appropriate neural network model for matrix inversion within these regions.

\subsection{Identifying Feasible Training Regions}
Recall Theorem \ref{main_theorem_pointwise}, which demonstrates that the norm $\|\text{Inv}(X)\|_L$ can become unbounded near singular matrices of rank $n - 1$. Therefore, when constructing a robust neural network-based matrix inversion model, it is crucial to ensure that the training data does not include matrices in close proximity to singular matrices.

To formalize this requirement, let $\epsilon > 0$ be a small positive number. We define the set $M_{\epsilon}$ as \begin{equation} M_{\epsilon} = \bigcup_{A \in S_M} B(A, \epsilon), \end{equation} where $S_M$ denotes the set of all singular matrices in $\mathbb{R}^{n \times n}$, and $B(A, \epsilon)$ represents the open ball in $\mathbb{R}^{n \times n}$ centered at $A$ with radius $\epsilon$. Consequently, to mitigate the challenges arising from the unbounded behavior of the inversion function near singular matrices—and considering practical limitations on data precision and numerical stability—it is essential that the training set excludes $M_{\epsilon}$.

To clarify this concept, we provide examples illustrating the $M_{\epsilon}$ region using 2D and 3D plots for the $2 \times 2$ matrix inversion problem. Let
\[
A = \begin{pmatrix}
    a_{11} & a_{12} \\
    a_{21} & a_{22}
\end{pmatrix}
\]
denote a $2 \times 2$ matrix. The determinant of $A$ is $\det(A) = a_{11}a_{22} - a_{12}a_{21}$. Thus, the set of singular matrices satisfies $\det(A) = 0$, and $M_{\epsilon}$ becomes
\begin{equation}
    M_{\epsilon} = \bigcup_{a_{11}a_{22} - a_{12}a_{21} = 0} B\left(
    \begin{pmatrix}
        a_{11} & a_{12} \\
        a_{21} & a_{22}
    \end{pmatrix}, 
    \epsilon\right).
\end{equation}

\begin{figure*}[!h]
  \centering
  \includegraphics[width=.45\linewidth]{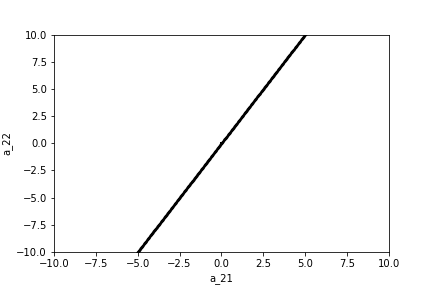}
  \hspace{1cm}
  \includegraphics[width=.45\linewidth]{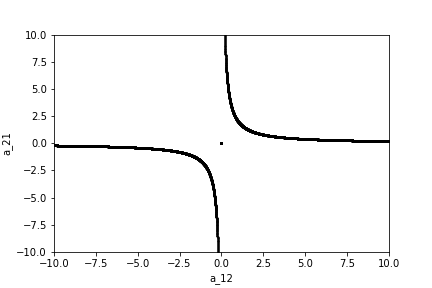}
  \caption{\textbf{Left}: The $M_\epsilon$ region (shaded area) for fixed $a_{11} = 1$, $a_{12} = 2$. \textbf{Right}: The $M_\epsilon$ region (shaded area) for fixed $a_{11} = 1$, $a_{22} = 2$.}
  \label{figure_22_2D}
\end{figure*}

\paragraph{Example 1: Fixed $a_{11} = 1$ and $a_{12} = 2$}

When $a_{11} = 1$ and $a_{12} = 2$ are fixed, $M_{\epsilon}$ consists of the $\epsilon$-neighborhoods around the line $a_{22} = 2a_{21}$ in the $(a_{21}, a_{22})$ plane. The shaded area in the left plot of Figure \ref{figure_22_2D} illustrates this $M_{\epsilon}$ region.

\paragraph{Example 2: Fixed $a_{11} = 1$ and $a_{22} = 2$}

When $a_{11} = 1$ and $a_{22} = 2$ are fixed, $M_{\epsilon}$ consists of the $\epsilon$-neighborhoods around the hyperbola $a_{12}a_{21} = 2$ in the $(a_{12}, a_{21})$ plane. The shaded area in the right plot of Figure \ref{figure_22_2D} illustrates this $M_{\epsilon}$ region.

\paragraph{Example 3: Fixed $a_{11} = 1$}

When only $a_{11} = 1$ is fixed, we can visualize the $M_{\epsilon}$ region in a 3D plot. The blue surface in Figure \ref{figure_22_3D} represents the set of matrices satisfying $\det(A) = 0$. The $M_{\epsilon}$ region consists of points near this surface. The yellow ball indicates a possible region for constructing the training data for accurate matrix inversion. 

\begin{figure}[!h]
  \centering
  \includegraphics[width=\linewidth]{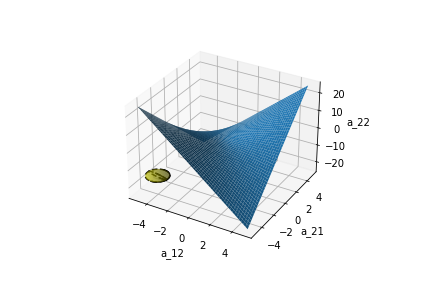}
  \caption{Blue surface: The $M_\epsilon$ region for fixed $a_{11} = 1$. Yellow ball: A training dataset area avoiding $M_\epsilon$.}
  \label{figure_22_3D}
\end{figure}

\subsection{Theoretical Analysis of Neural Network Approximation}\label{theoretical_model_design}

In this section, we theoretically describe how to design the coefficients of a neural network to approximate the matrix inversion in a specific region $M_0 \subset \mathbb{R}^{n\times n}$. The region $M_0$ is defined as:
\begin{equation}
M_0 = \prod_{i=1}^{n}\prod_{j=1}^{n}[A_{0,i,j} - c, \, A_{0,i,j} + c],
\end{equation}
where $A_0 \in \mathbb{R}^{n \times n}$ is a fixed nonsingular matrix, $A_{0,i,j}$ denotes the $(i,j)$th entry of $A_0$ and $c > 0$ is a constant. We assume that $M_0$ does not intersect with $M_\epsilon$, i.e., $M_0 \cap M_\epsilon = \emptyset$.

Let $A = A_0 + A'$ denote a matrix in $M_0$. Clearly, $A' \in \prod_{i=1}^{n}\prod_{j=1}^{n}[-c, \, c]$.  Consider the determinant formula:
\begin{equation}
    \operatorname{det}(A) = \sum\limits_{\sigma \in S_n} (\text{sgn}(\sigma)\prod_{i=1}^n A_{i, \sigma(i)})
\end{equation}
where $\text{sgn}$ denotes the permutation function, $S_n$ denotes the symmetric group of all such permutations, $\text{sgn}(\sigma)$ is $1$ if the permutation can be obtained from even number of exchanges of two entries, $-1$ otherwise.

For small perturbations $A'$, each element of the inverse matrix $(A_0 + A')^{-1}_{k,l}$ can be expressed as:
\begin{equation}\label{inv_formula_ori}
\small
\begin{split}
    (-1)^{k+l}\frac{\sum\limits_{\sigma \in S_{n-1}} (\text{sgn}(\sigma)\prod_{i=1}^{n-1} (A^{adj, k,l}_{0, i, \sigma(i)}+A'^{,adj, k,l}_{i,\sigma(i)}))}{\sum\limits_{\sigma \in S_n} (\text{sgn}(\sigma)\prod_{i=1}^n (A_{0,i, \sigma(i)}+A'_{i,\sigma(i)}))}
\end{split}
\end{equation}
where $A^{adj, k,l}_{0, i, \sigma(i)}$ denotes the $(i,\sigma(i))$th entry of the matrix after deleting row $k$ and column $l$ from $A_0$.

\begin{table*}[ht]
  \centering
  \small
  \caption{Average absolute error on the test set. Numbers in parentheses are standard deviations. The average and standard deviation are computed over 3 runs.}
  \label{experiment_main}
  \vskip.05in
  \renewcommand{\arraystretch}{1.3}
  \begin{tabular}{c|cc}
    \toprule
    Model & $2\times 2$ (First) & $2\times 2$ (Second)\\
    \midrule
    2-Fully Connected Layers & $2.05\times 10^{-5} \ (3.78\times 10^{-11})$ & $6.39\times 10^{-6} \ (3.72\times 10^{-12})$ \\
    3-Fully Connected Layers & $1.52\times 10^{-5} \ (4.31\times 10^{-11})$ & $7.37\times 10^{-6} \ (3.95\times 10^{-12})$ \\
    \midrule
    Model & $3\times 3$ Dataset & $16\times 16$ Dataset \\
    \midrule
    2-Fully Connected Layers & $8.77\times 10^{-5} \ (9.40\times 10^{-11})$ & $1.68\times 10^{-4} \ (3.19\times 10^{-10})$\\
    3-Fully Connected Layers & $1.53\times 10^{-4} \ (4.41\times 10^{-9})$ & $2.58\times 10^{-4} \ (8.18\times 10^{-11})$\\
    \bottomrule
  \end{tabular}
\end{table*}

Hence, we can show that 
\begin{equation}\label{2_layer_formula}
\begin{split}
(A_0 + A')^{-1}_{k,l} &= f^0_{k,l}(\{A_{0}\})+  \\
\sum_{i,j=1}^{n}&f^1_{k,l,i,j}(\{A_{0}\})A'_{i,j} + O(\{A'_{i,j}\}^2),
\end{split}
\end{equation}
where $f^{idx}_{k,l,i,j}(\{A_{0}\})$ represents a function of all elements in $A_0$, and $O(\{A'_{i,j}\}^2)$ represents the higher-order term.

As a result, we can design a neural network with 2 fully connected layers to approximate formula \ref{2_layer_formula} as:
\begin{equation}
    \begin{split}
        &h^{1+}_{k,l} = \text{ReLU}(\sum_{i,j=1}^{n}f^1_{k,l,i,j}(\{A_{0}\})A'_{i,j})\\
        &h^{1-}_{k,l} = \text{ReLU}(\sum_{i,j=1}^{n}(-f^1_{k,l,i,j}(\{A_{0}\}))A'_{i,j})\\
        &h^{2}_{k,l} = f^0_{k,l}(\{A_{0}\}) + 1 \times h^{1+}_{k,l} + 1 \times h^{1-}_{k,l}\\
    \end{split}
\end{equation}
and the error is $O(\{A'_{i,j}\}^2)$. When $A'$ is small enough, the error term decreases quadratically.

Therefore, for the design of a neural network-based end-to-end matrix inversion model in a specific region, with $n^2$ input elements, two fully connected layers with at least $2n^2$ hidden units can perform as well as the linear approximation.

\iffalse
\subsection{Theoretical Analysis of Coefficients in Neural-based Model}\label{theoretical_model_design}

In this section, we describe how to design the coefficients of the neural network to approximate the matrix inversion operation in a specific area $M_0 = \prod_{i=1,j=1}^{n}[A_{0,i,j} - c,A_{0,i,j} + c] \in \mathbb{R}^{n\times n}$ theoretically. We assume that $M_0\cap M_\epsilon = \emptyset$.
\fi

%%%%%%%%%%%%%%%%%%%%%%%%%%%%%%%%%%%%%%%%%%%%%%%%%%%%%%%%%%%%
\section{Experiments}
In this section, we present experiments that support our theoretical analysis. We train neural network models on four different matrix inversion datasets and present the results in Section \ref{experiment_result}. In Section \ref{experiment_theoretical}, we compare a two-layer model with a small number of hidden units to a linear approximation, demonstrating experimentally and theoretically what the models learn.

\subsection{Experiment Setup}
\textbf{Datasets}

We conducted experiments on matrices of various sizes: $2 \times 2$, $3 \times 3$, and $16 \times 16$. For quick verification and parameter tuning, we used small matrices ($2 \times 2$ and $3 \times 3$). To demonstrate that our statements hold for larger matrices, we experimented with $16 \times 16$ matrices.

For the $2 \times 2$ matrices, to validate our theorem regarding neighborhoods with no intersection with $M_{\epsilon}$, we generated two datasets centered at
\[
\begin{pmatrix}
2 & 2 \\
2 & 3
\end{pmatrix}
\quad \text{and} \quad
\begin{pmatrix}
2 & 1 \\
0 & -1
\end{pmatrix},
\]
respectively. Each dataset was constructed as $\prod_{i=1}^{2}\prod_{j=1}^{2}[A_{0,i,j} - 0.01, A_{0,i,j} + 0.01] \subset \mathbb{R}^{2 \times 2}$. The first dataset centers around a positive definite symmetric matrix, while the second centers around a general matrix. We denote them as $2 \times 2$ (First) and $2 \times 2$ (Second) datasets.

For $3 \times 3$ matrices, we generated a dataset centered at
\[
\begin{pmatrix}
1 & 1 & 1\\
1 & 2 & 3\\
1 & 2 & 4\\
\end{pmatrix}.
\]

For $16 \times 16$ matrices, we selected a non-singular matrix with elements sampled from $\{-2, -1, 0, 1, 2\}$ and generated the dataset by sampling around this matrix.

\textbf{Implementation Details}

We used neural networks with 2 or 3 fully connected layers and ReLU activation functions, setting hidden dimensions to several times the input dimension. The optimizer used was Adam, and the loss function was mean squared error (MSE). Hyperparameters were selected via grid search, and the learning rate warm restart technique \citep{warmrestart17} was applied. Details are provided in Appendices \ref{hyperparameter_appendix} and \ref{hyperparameter_appendix_theoretical}. Experiments were conducted on Nvidia GPUs.

\subsection{Results}\label{experiment_result}

We performed experiments demonstrating that neural networks with 2 or 3 fully connected layers can approximate matrix inversion. Table \ref{experiment_main} shows the average absolute error between the neural network output and the true inverse on the test set, averaged over 3 runs.

For $2 \times 2$ matrices, the neural network's inverse elements are approximately $10^{-5}$ away from the ground truth, indicating effective learning within the dataset's space. For $3 \times 3$ matrices, the error increases to $10^{-4}$, and for $16 \times 16$ matrices, the error is about $2 \times 10^{-4}$, which is still relatively small.

All the inference times of trained models are listed in Appendix \ref{inference_time_section}.

\subsection{Comparison with Linear Approximation}\label{experiment_theoretical}

We trained a small 2-layer model on the $2 \times 2$ (first) dataset and analyzed the model to understand what it learned about matrix inversion.

First, we introduce the linear approximation of matrix inversion around
$\begin{pmatrix}
2 & 2 \\
2 & 3
\end{pmatrix}$. Let the input matrix be in the form of$\begin{pmatrix}
2 & 2 \\
2 & 3
\end{pmatrix}+\begin{pmatrix}
a & b \\
c & d
\end{pmatrix}$, and the inversion has the form of $\begin{pmatrix}
1.5 & -1 \\
-1 & 1
\end{pmatrix}+\begin{pmatrix}
a_{11} & a_{12} \\
a_{21} & a_{22}
\end{pmatrix}$. From formula \ref{inv_formula_ori}, we can give the linear approximation of $a_{ij}$ as
\begin{equation}\label{linear_approxiamtion_formula}
    \begin{split}
        &a_{11}\approx - 2.25a+ 1.5b+ 1.5c- d\\
        &a_{12}\approx 1.5a- 1.5b-c+d\\
        &a_{21}\approx 1.5a- b-1.5c+d\\
        &a_{22}\approx  - a + b+ c-  d\\
    \end{split}
\end{equation}

Then, we compare the linear approximation method with the neural-network-based method. In Table \ref{experiment_compare_linear}, we compare the average absolute error on the test set for four different models. It is obvious that neural-network models have better performance than the traditional linear approximation method. 
\begin{table}[H]
  \small 
  %\centering
  \caption{Average absolute error on $2\times 2$ (First) test set for different models. For the deep learning model, the average is computed by 3 runs.}
  \label{experiment_compare_linear}
  % \vskip.05in
  %     \renewcommand{\arraystretch}{1.3}
  \begin{tabular}{c|c}
    \toprule
    Model & Average absolute error\\
    \midrule
    Linear Approximation & $1.97\times 10^{-4}$\\
    2-Fully Connect(small) & $6.82\times 10^{-5}$\\
    2-Fully Connect & $2.05\times 10^{-5}$\\
    3-Fully Connect & $1.52\times 10^{-5}$\\
    \bottomrule
  \end{tabular}
\end{table}

Because the neural network can be written as $W_2\text{ReLU}(W_1(a,b,c,d)^T+b_1)+b_2$, we try to compute what the formula represents in different spaces in $\mathbb{R}^4$. 
We first trained a 2-layers neural network and fixed the parameters after training. Then, we randomly sampled 1M data points in the dataset area, and found that $55.7\%$ of the sampled data located in the area $\{h_i>0|i \in \{1,4,5,6,7\}\}\cap \{h_i<0|i \in \{3\}\}$, where $h_i$ represents the hidden unit in layer 1 of the fixed neural network. If we eliminate the ReLU function for $h_i>0$ and discard the negative unit, the output $a_{ij}$ of the neural network in this area has the form
\begin{equation}\label{neural_network_approxiamtion_formula_1}
\begin{split}
\small
     a_{11} = & -2.3034*a + 1.5408*b  \\
              & + 1.5354*c- 1.0260*d - 0.0102 \\
     a_{12} =&  1.5392*a - 1.5324*b  \\
              & - 1.0302*c + 1.0241*d + 0.0081 \\
     a_{21} =& 1.5373*a - 1.0313*b  \\
              & - 1.5265*c + 1.0220*d + 0.0060\\
     a_{22} =& -1.0290*a + 1.0248*b   \\
              &+ 1.0215*c - 1.0180*d - 0.0049\\
\end{split}
\end{equation}
We find that the difference between each coefficient in formula ~\eqref{neural_network_approxiamtion_formula_1} and the corresponding coefficient in the linear approximation formula~\eqref{linear_approxiamtion_formula} is less than $0.06$, indicating they are nearly identical to the linear approximation. 

For other data $41.7\%$, located in $\{h_i>0|i \in \{1,3,4,5\}\}\cap \{h_i<0|i \in \{6,7\}\}$, we analyze the neural network in Appendix \ref{appendix_full_analysis}. We find that the difference between each coefficient in the formula and the linear approximation is less than $0.07$.
%the output $a_{ij}$ has the form
% \begin{equation}
% \begin{split}
% \small
%     a_{11} =& -2.1860*a + 1.4526*b   \\
%             & + 1.4583*c - 0.9698*d - 0.0229\\
%     a_{12} =&  1.4544*a - 1.4624*b   \\
%             & - 0.9641*c + 0.9717*d + 0.0174\\
%     a_{21} =&  1.4550*a - 0.9621*b  \\
%             & - 1.4679*c + 0.9733*d + 0.0167\\
%     a_{22} =& -0.9652*a + 0.9702*b  \\
%             & + 0.9741*c - 0.9783*d - 0.0131\\
% \end{split}
% \end{equation}

These two cases cover most of the sampled data ($97.4\%$), leading us to conclude that, in most of the dataset, the two-layer model essentially learns a refined linear approximation of $a_{ij}$. Appendix~\ref{appendix_full_analysis} provides a full analysis of all $a_{ij}$ in all cases and examines the neural network's properties for the remaining $2.6\%$ of data.

%%%%%%%%%%%%%%%%%%%%%%%%%%%%%%%%%%%%%%%%%%%%%%%%%%%%%%%%%%%%
\section{Discussion and Conclusions}

\subsection{Limitations}

First, our techniques for estimating the approximation error bounds cannot establish that the approximation error over a general set must be large. Estimating the bounds of the mean squared error in high-dimensional spaces when approximating matrix inversion with neural networks requires further analysis.

Second, due to limited computational resources, we did not experiment with very large matrices (e.g., $10{,}000 \times 10{,}000$) or deeper neural networks.

Finally, we did not investigate the performance of non-polynomial Lipschitz networks on the matrix inversion problem.

\subsection{Conclusions and future work}

We proved that most modern neural network structures cannot form a general matrix inversion model. To support our proofs, we defined a generalized Lipschitz function class that more accurately describes modern neural networks. 

Future work could explore the performance of non-polynomial Lipschitz networks on mathematical tasks and develop a more comprehensive function class encompassing more neural network structures to analyze their capabilities. We also identified regions where neural networks can effectively approximate matrix inversion, both theoretically and experimentally. This insight may help elucidate what black-box neural networks learn in specific tasks.

\section{Acknowledgements}
The Research of Y. Xi is supported by the National
Science Foundation (NSF) under Grant No. DMS-2338904.

\newpage
\bibliography{neural_based_matrix_inversion}

\begin{thebibliography}{21}
\providecommand{\natexlab}[1]{#1}
\providecommand{\url}[1]{\texttt{#1}}
\expandafter\ifx\csname urlstyle\endcsname\relax
  \providecommand{\doi}[1]{doi: #1}\else
  \providecommand{\doi}{doi: \begingroup \urlstyle{rm}\Url}\fi

\bibitem[Almasadeh et~al.(2022)Almasadeh, Alnajjar, and Albreem]{10019100}
Ali~J. Almasadeh, Khawla~A. Alnajjar, and Mahmoud~A. Albreem.
\newblock Enhanced deep learning for massive mimo detection using approximate matrix inversion.
\newblock In \emph{2022 5th International Conference on Communications, Signal Processing, and their Applications (ICCSPA)}, 2022.

\bibitem[Anil et~al.(2019)Anil, Lucas, and Grosse]{anil2019sorting}
Cem Anil, James Lucas, and Roger Grosse.
\newblock Sorting out lipschitz function approximation.
\newblock In \emph{Proceedings of the 36th International Conference on Machine Learning (ICML)}, volume~97, pages 291--301. PMLR, 2019.

\bibitem[Ben-Shaul et~al.(2023)Ben-Shaul, Bar, Fishelov, and Sochen]{10.1162/neco_a_01583}
Ido Ben-Shaul, Leah Bar, Dalia Fishelov, and Nir Sochen.
\newblock {Deep Learning Solution of the Eigenvalue Problem for Differential Operators}.
\newblock \emph{Neural Computation}, 35\penalty0 (6):\penalty0 1100--1134, 05 2023.

\bibitem[Cai et~al.(2022)Cai, Ji, He, Ye, and Xi]{autm22cai}
Difeng Cai, Yuliang Ji, Huan He, Qiang Ye, and Yuanzhe Xi.
\newblock Autm flow: Atomic unrestricted time machine for monotonic normalizing flows.
\newblock In \emph{UAI}, 2022.

\bibitem[Dai et~al.(2022)Dai, Chen, Xiao, Jia, and He]{9468339}
Jianhua Dai, Yuanmeng Chen, Lin Xiao, Lei Jia, and Yongjun He.
\newblock Design and analysis of a hybrid gnn-znn model with a fuzzy adaptive factor for matrix inversion.
\newblock \emph{IEEE Transactions on Industrial Informatics}, 18\penalty0 (4):\penalty0 2434--2442, 2022.

\bibitem[Dai et~al.(2023)Dai, Tan, Xiao, Jia, He, and Luo]{10115024}
Jianhua Dai, Ping Tan, Lin Xiao, Lei Jia, Yongjun He, and Jiajie Luo.
\newblock A fuzzy adaptive zeroing neural network model with event-triggered control for time-varying matrix inversion.
\newblock \emph{IEEE Transactions on Fuzzy Systems}, 31\penalty0 (11):\penalty0 3974--3983, 2023.

\bibitem[Dongarra et~al.(1990)Dongarra, Duff, Sorensen, and van~der Vorst]{dongarra1990solving}
Jack~J. Dongarra, Iain~S. Duff, Danny~C. Sorensen, and Henk~A. van~der Vorst.
\newblock Solving linear systems on vector and shared memory computers.
\newblock \emph{SIAM Journal on Scientific and Statistical Computing}, 11\penalty0 (5):\penalty0 1096--1127, 1990.

\bibitem[Durkan et~al.(2019)Durkan, Bekasov, Murray, and Papamakarios]{neuralspline19}
Conor Durkan, Artur Bekasov, Iain Murray, and George Papamakarios.
\newblock Neural spline flows.
\newblock In \emph{Advances in Neural Information Processing Systems}, pages 7509--7520, 2019.

\bibitem[Fa-Long and Zheng(1992)]{82neural_matrix_inv82luo}
Luo Fa-Long and Bao Zheng.
\newblock Neural network approach to computing matrix inversion.
\newblock \emph{Applied Mathematics and Computation}, 47:\penalty0 109--120, 1992.

\bibitem[Fa-Long and Zheng(2009)]{newtoniter09zhang}
Luo Fa-Long and Bao Zheng.
\newblock From zhang neural network to newton iteration for matrix inversion.
\newblock \emph{IEEE TRANSACTIONS ON CIRCUITS AND SYSTEMS—I}, 56\penalty0 (7):\penalty0 1405--1414, 2009.

\bibitem[Fawzi et~al.(2022)Fawzi, Balog, Huang, Hubert, Romera-Paredes, Barekatain, Novikov, Ruiz, Schrittwieser, Swirszcz, Silver, Hassabis, and Kohli]{matrixmultiplication22Nature}
Alhussein Fawzi, Matej Balog, Aja Huang, Thomas Hubert, Bernardino Romera-Paredes, Mohammadamin Barekatain, Alexander Novikov, Francisco J.~R. Ruiz, Julian Schrittwieser, Grzegorz Swirszcz, David Silver, Demis Hassabis, and Pushmeet Kohli.
\newblock Discovering faster matrix multiplication algorithms with reinforcement learning.
\newblock \emph{Nature}, 610:\penalty0 47--53, 2022.

\bibitem[Gerontitis et~al.(2023)Gerontitis, Mo, Stanimirović, Tzekis, and Katsikis]{Dimitrios2023}
Dimitrios Gerontitis, Changxin Mo, Predrag~S. Stanimirović, Panagiotis Tzekis, and Vasilios~N. Katsikis.
\newblock A novel extended li zeroing neural network for matrix inversion.
\newblock \emph{Neural Computing and Applications}, 35:\penalty0 14129–14152, 2023.

\bibitem[Golub and Van~Loan(1996)]{golub1996matrix}
Gene~H. Golub and Charles~F. Van~Loan.
\newblock \emph{Matrix Computations}.
\newblock Johns Hopkins University Press, Baltimore, MD, 3rd edition, 1996.

\bibitem[Han et~al.(2020)Han, Lu, and Zhou]{HAN2020109792}
Jiequn Han, Jianfeng Lu, and Mo~Zhou.
\newblock Solving high-dimensional eigenvalue problems using deep neural networks: A diffusion monte carlo like approach.
\newblock \emph{Journal of Computational Physics}, 423:\penalty0 109792, 2020.
\newblock ISSN 0021-9991.

\bibitem[Jang et~al.(1987)Jang, Lee, and Shin]{NIPS1987_6c8349cc}
Ju-Seog Jang, Soo-Young Lee, and Sang-Yung Shin.
\newblock An optimization network for matrix inversion.
\newblock In \emph{Neural Information Processing Systems}, 1987.

\bibitem[Kim et~al.(2021)Kim, Papamakarios, and Mnih]{lipschitz21Kim}
Hyunjik Kim, George Papamakarios, and Andriy Mnih.
\newblock The lipschitz constant of self-attention.
\newblock In \emph{ICML}, 2021.

\bibitem[Latorre et~al.(2020)Latorre, Rolland, , and Cevher]{lipschitz20Larorre}
F.~Latorre, P.~Rolland, , and V.~Cevher.
\newblock Lipschitz constant estimation of neural networks via sparse polynomial optimization.
\newblock In \emph{International Conference on Learning Representations}, 2020.

\bibitem[Li and Hu(2022)]{https://doi.org/10.1049/sil2.12156}
Lin Li and Jianhao Hu.
\newblock An efficient second-order neural network model for computing the moore–penrose inverse of matrices.
\newblock \emph{IET Signal Processing}, 16\penalty0 (9):\penalty0 1106--1117, 2022.

\bibitem[Loshchilov and Hutter(2017)]{warmrestart17}
Ilya Loshchilov and Frank Hutter.
\newblock {SGDR:} stochastic gradient descent with warm restarts.
\newblock In \emph{ICLR}, 2017.

\bibitem[Steriti et~al.(1990)Steriti, Coleman, and Fiddy]{Steriti90matrix_inv}
R.~Steriti, J.~Coleman, and M.~A. Fiddy.
\newblock A neural network based matrix inversion algorithm.
\newblock In \emph{IJCNN}, 1990.

\bibitem[Virmaux and Scaman(2018)]{lipschitz18virmaux}
A.~Virmaux and K.~Scaman.
\newblock Lipschitz regularity of deep neural networks: analysis and efficient estimation.
\newblock In \emph{Advances in Neural Information Processing Systems}, 2018.

\end{thebibliography}
\bibliographystyle{plainnat}

%%%%%%%%%%%%%%%%%%%%%%%%%%%%%%%%%%%%%%%%%%%%%%%%%%%%%%%%%%%%
\section*{Checklist}

% %%% BEGIN INSTRUCTIONS %%%
%The checklist follows the references. For each question, choose your answer from the three possible options: Yes, No, Not Applicable.  You are encouraged to include a justification to your answer, either by referencing the appropriate section of your paper or providing a brief inline description (1-2 sentences). 
%Please do not modify the questions.  Note that the Checklist section does not count towards the page limit. Not including the checklist in the first submission won't result in desk rejection, although in such case we will ask you to upload it during the author response period and include it in camera ready (if accepted).

%\textbf{In your paper, please delete this instructions block and only keep the Checklist section heading above along with the questions/answers below.}
% %%% END INSTRUCTIONS %%%

 \begin{enumerate}

 \item For all models and algorithms presented, check if you include:
 \begin{enumerate}
   \item A clear description of the mathematical setting, assumptions, algorithm, and/or model. [Yes]
   \item An analysis of the properties and complexity (time, space, sample size) of any algorithm. [Yes]
   \item (Optional) Anonymized source code, with specification of all dependencies, including external libraries. [Yes]
 \end{enumerate}

 \item For any theoretical claim, check if you include:
 \begin{enumerate}
   \item Statements of the full set of assumptions of all theoretical results. [Yes]
   \item Complete proofs of all theoretical results. [Yes]
   \item Clear explanations of any assumptions. [Yes]     
 \end{enumerate}

 \item For all figures and tables that present empirical results, check if you include:
 \begin{enumerate}
   \item The code, data, and instructions needed to reproduce the main experimental results (either in the supplemental material or as a URL). [Yes]
   \item All the training details (e.g., data splits, hyperparameters, how they were chosen). [Yes]
         \item A clear definition of the specific measure or statistics and error bars (e.g., with respect to the random seed after running experiments multiple times). [Yes]
         \item A description of the computing infrastructure used. (e.g., type of GPUs, internal cluster, or cloud provider). [Yes]
 \end{enumerate}

 \item If you are using existing assets (e.g., code, data, models) or curating/releasing new assets, check if you include:
 \begin{enumerate}
   \item Citations of the creator If your work uses existing assets. [Not Applicable]
   \item The license information of the assets, if applicable. [Yes]
   \item New assets either in the supplemental material or as a URL, if applicable. [Not Applicable]
   \item Information about consent from data providers/curators. [Not Applicable]
   \item Discussion of sensible content if applicable, e.g., personally identifiable information or offensive content. [Not Applicable]
 \end{enumerate}

 \item If you used crowdsourcing or conducted research with human subjects, check if you include:
 \begin{enumerate}
   \item The full text of instructions given to participants and screenshots. [Not Applicable]
   \item Descriptions of potential participant risks, with links to Institutional Review Board (IRB) approvals if applicable. [Not Applicable]
   \item The estimated hourly wage paid to participants and the total amount spent on participant compensation. [Not Applicable]
 \end{enumerate}

 \end{enumerate}

%%%%%%%%%%%%%%%%%%%%%%%%%%%%%%%%%%%%%%%%%%%%%%%%%%%%%%%%%%%%%%%%%%%%%%%%%%%%%%%
%%%%%%%%%%%%%%%%%%%%%%%%%%%%%%%%%%%%%%%%%%%%%%%%%%%%%%%%%%%%%%%%%%%%%%%%%%%%%%%
% APPENDIX
%%%%%%%%%%%%%%%%%%%%%%%%%%%%%%%%%%%%%%%%%%%%%%%%%%%%%%%%%%%%%%%%%%%%%%%%%%%%%%%
%%%%%%%%%%%%%%%%%%%%%%%%%%%%%%%%%%%%%%%%%%%%%%%%%%%%%%%%%%%%%%%%%%%%%%%%%%%%%%%
\newpage
\appendix
\onecolumn
\section{Supplement proofs}\label{Appendix_proofs}

\subsection{Proof of Lemma \ref{lemma_composition_functions}}\label{Appendix_A_1}

Recall the lemma
\begin{lemma}
    Suppose we have two functions $f(x):\mathbb{R}^{n_2}\rightarrow \mathbb{R}^{n_3}$, $g(x):\mathbb{R}^{n_1}\rightarrow \mathbb{R}^{n_2}$. And, Under norm $L$ (either $L_1$, $L_2$ or $L_\infty$, defined on $\mathbb{R}^{n_i}$), $f,g$ are all polynomial Lipschitz continuous functions. Then the composition $f\circ g$ is also a polynomial Lipschitz continuous function.
\end{lemma}

\begin{proof}

First we prove that any polynomial Lipschitz continuous function $g(x)$ is bounded by a polynomial of $\|x\|_{L^1}$ under $L^2$ norm in $\mathbb{R}^{n_2}$. Recall the definition
\begin{equation}
    \|f(x)-f(y)\|_{L^2}\leq \sum\limits_{i=0}^{n^f} f_i(\|x\|_{L^1}, \|y\|_{L^1}) \|x-y\|_{L^1}^i,
\end{equation}
we have
\begin{equation}
\begin{split}
    \|g(x)\|_{L^2} & \leq \|g(x)-g(0)\|_{L^2} + \|g(0)\|_{L^2}\\
    &\leq \sum\limits_{i=0}^{n^g} g_i(\|x\|_{L^1}, \|0\|_{L^1}) \|x\|_{L^1}^i +  \|g(0)\|_{L^2}\\
    &= \sum\limits_{i=0}^{n^g} g_i(\|x\|_{L^1}, 0) \|x\|_{L^1}^i +  \|g(0)\|_{L^2}\\
\end{split}
\end{equation}
where $g_i(\|x\|_{L^1}, 0)$ can downgrades to a polynomial with only one variable $\|x\|_{L^1}$. As a result, $\|g(x)\|_{L^2}$ is bounded by a polynomial of $\|x\|_{L^1}$.

Then consider 
\begin{equation}
\begin{split}
    \|f(g(x))-f(g(y))\|_{L}&\leq \sum\limits_{i=0}^{n^f} f_i(\|g(x)\|_{L^1}, \|g(y)\|_{L^1}) \|g(x)-g(y)\|_{L}^i\\
    &\leq \sum\limits_{i=0}^{n^f} f_i(\|g(x)\|_{L^1}, \|g(y)\|_{L^1}) (\sum\limits_{j=0}^{n^g} g_j(\|x\|_{L^1}, \|y\|_{L^1}) \|x-y\|_{L}^j)^i\\
    &\leq \sum\limits_{i=0}^{n^f} ploynomial^g_{f_i}(\|x\|_{L^1}, \|y\|_{L^1}) (\sum\limits_{j=0}^{n^g} g_j(\|x\|_{L^1}, \|y\|_{L^1}) \|x-y\|_{L}^j)^i\\
    &= \sum\limits_{i=0}^{n^f\times n^g} (f\circ g)_i(\|x\|_{L^1}, \|y\|_{L^1}) \|x-y\|_{L}^i
\end{split}
\end{equation}

where $(f\circ g)_i(\cdot,\cdot)$ can be calculated from $f_i(\cdot,\cdot)$, $g_i(\cdot,\cdot)$ and the upper bound polynomial of $\|g\|_{L^2}$. The last inequality are from the property of polynomials.
\end{proof}

\subsection{Proof of Lemma \ref{lemma_parallel_functions}}\label{Appendix_A_2}

Recall the lemma
\begin{lemma}
    Suppose we have two functions $f(x):\mathbb{R}^{n_1}\rightarrow \mathbb{R}^{n_2}$, $g(x):\mathbb{R}^{n_1}\rightarrow \mathbb{R}^{n_3}$. And, under norm $L$, either $L_1$, $L_2$ or $L_\infty$, defined on $\mathbb{R}^{n_i}$, $f,g$ are all polynomial Lipschitz continuous functions. Then the combination $(f,g)(x):\mathbb{R}^{n_1}\rightarrow \mathbb{R}^{n_2+n_3}$ is also a polynomial Lipschitz continuous function.
\end{lemma}

\begin{proof}
\begin{equation}
\begin{split}
    \|(f,g)(x)- (f,g)(y)\|_{L}& = \|(f(x)-f(y), g(x)-g(y))\|_{L}\\
    &\leq \|f(x)-f(y)\|_{L} + \|g(x)-g(y)\|_{L}\\
    &\leq \sum\limits_{i=0}^{n^f}  f_i(\|x\|_{L^1}, \|y\|_{L^1}) \|x-y\|_{L}^i + \sum\limits_{i=0}^{n^g}  g_i(\|x\|_{L^1}, \|y\|_{L^1}) \|x-y\|_{L}^i\\
    &= \sum\limits_{i=0}^{n^{(f,g)}}(f_i+g_i)(\|x\|_{L^1}, \|y\|_{L^1}) \|x-y\|_{L}^i.
\end{split}
\end{equation}
\end{proof}

\subsection{Proof of Lemma \ref{lemma_lipschitz_Jacobian}}\label{Appendix_A_Jacobian}
Recall the Lemma
\begin{lemma}
    Suppose function $f(x):\mathbb{R}^{n_1}\rightarrow \mathbb{R}^{n_2}$ and Jacobian of $f$ exists everywhere. If the value of each element of the Jacobian is bounded by a polynomial of $\|x\|_L$, $f$ is a polynomial Lipschitz continuous function.
\end{lemma}

\begin{proof}
    Here, $L$ represents either $L_1, L_2$ or $L_{\infty}$. We represent $f(x)$ as $f(x)=(f_1(x), f_2(x),...,f_{n_2}(x))$, where $f_i(x):\mathbb{R}^{n_1}\rightarrow \mathbb{R}$. Define $f_{k,xy}(t) = f_k(tx+(1-t)y)$ for $0\leq t\leq 1$, which is a continuous function. From the definition of $f(x)$, $f_{k,xy}(t)$ has derivative, and is bounded by a polynomial of $\|tx+(1-t)y\|_L$.
    
    Then we have
    \begin{equation}
        \begin{split}
            \|f(x)-f(y)\|_{L^2}&=\|(|f_1(x)-f_1(y)|, |f_2(x)-f_2(y)|,..., |f_{n_2}(x)-f_{n_2}(y)|) \|_{L^2}\\ 
            &= \|(|f_{1, xy}(1) - f_{1, xy}(0)|,|f_{2, xy}(1) - f_{2, xy}(0)|,..., |f_{{n_2}, xy}(1) - f_{k, xy}(0)|  \|_{L^2}\\
            &= \|(|f'_{1,xy}(\epsilon_1)|, |f'_{2,xy}(\epsilon_2)|,..., |f'_{{n_2},xy}(\epsilon_{n_2}))| \|_{L^2} \ (Mean \ Value \ Theorem)\\
            &\leq \|(poly_{Jacobian, 1}(\|\epsilon_1 x+(1-\epsilon_1)y\|_{L^1}), ..., poly_{Jacobian, {n_2}}(\|\epsilon_{n_2} x+(1-\epsilon_{n_2})y\|_{L^1}))\|_{L^2}\\
            &\leq \|(poly_{Jacobian, 1}(\|y\|_{L^1} + \epsilon_1\|y-x\|_{L^1}), ..., poly_{Jacobian, {n_2}}(\|y\|_{L^1} + \epsilon_{n_2}\|y-x\|_{L^1}))\|_{L^2}\\
            &\leq \|(\sum\limits_{i=0}^{F_1} poly_{Jacobian, 1, i}(\|y\|_{L^1}) \|y-x\|^i_{L^1}, ..., \sum\limits_{i=0}^{F_{n_2}} poly_{Jacobian, {n_2}, i}(\|y\|_{L^1}) \|y-x\|^i_{L^1})\|_{L^2}\\
            &\leq \sum\limits_{i=0}^{max(\{F_i\})} \|(poly_{Jacobian, 1, i}(\|y\|_{L^1}), ..., poly_{Jacobian, {n_2}, i}(\|y\|_{L^1}))\|_{L^2} \|y-x\|^i_{L^1}\\
            &\leq \sum\limits_{i=0}^{max(\{F_i\})} poly^f_{Jacobian, i}(\|y\|_{L^1}) \|y-x\|^i_{L^1}\\
        \end{split}
    \end{equation}
\end{proof}

\subsection{Proof of Lemma \ref{lemma_lipschitz_neural_network}}\label{Appendix_A_3}

Recall the lemma
\begin{lemma}
    Some modern widely used neural structures are polynomial Lipschitz continuous functions.
\end{lemma}

\begin{proof}

We list the most widely used neural network structures below:
\begin{itemize}
    \item Fully Connect Layer: The formula is $Wx+b$, obviously, it is a Lipschitz continuous function, which is also a polynomial Lipschitz continuous.
    \item FCN Layer, CNN Layer, Non-linearities (relu, sigmoid, tanh): Under the choices of our norm, they are Lipschitz functions\citep{lipschitz21Kim}. Hence, they are polynomial Lipschitz continuous functions.
    \item Neural Spline Layer: In \citep{neuralspline19, autm22cai}, they introduce element-wise polynomial layers containing quadratic term and cubic term. For $f(x)=x^n$, obviously we have 
    \begin{equation}
    \begin{split}
        |f(x)-f(y)|&=|x^n-y^n| = |x-y\|\sum\limits_{i=0}^{n-1}x^iy^{n-1-i}|\\
        &\leq (\sum\limits_{i=0}^{n-1}|x|^i|y|^{n-1-i})|x-y|,
    \end{split}
    \end{equation}
    As a result, these layers are polynomial Lipschitz continuous functions.
    \item Residual structure: It has the form of $y = x+f(x)$. If $f(x)$ is polynomial Lipschitz function, satisfying $\|f(x_1)-f(x_2)\|_{L^2}\leq \sum\limits_{i=0}^{n^f} f_i(\|x_1\|_{L^1}, \|x_2\|_{L^1}) \|x_1-x_2\|_{L^1}^i$, then $\|y_1-y_2\|_{L^2}=\|x_1 - x_2 + f(x_1)-f(x_2)\|_{L^2}\leq \sum\limits_{i=0}^{n^f} f_i(\|x_1\|_{L^1}, \|x_2\|_{L^1})  \|x_1-x_2\|_{L^1}^i + \|x_1-x_2\|_{L^1}$.
    \item RNN, LSTM Unit: RNN unit is a combination of matrix multiplication, tanh function, and softmax function. Because they are all Lipschitzable, RNN is a Lipschitz function. Because the LSTM unit is a combination of matrix multiplication, tanh function, and activation functions, its polynomial Lipschitzable depends on the polynomial Lipschitz continuity of the used activation functions.
    \item Attention Layer: 
    %Previous research \citep{lipschitz21Kim} showed that multi-head dot-product attention is not Lipschitzable. However, when we investigate the Jacobian matrix in their paper, the entry of the Jacobian is $J_{ij} =  X^T P^{(i)}[E_{ji}XA^T+XA\delta_{ij}]+P_{ij}I$. Although it can be extremely large for large $\|X\|_{L^p}$, the entry is bounded by a polynomial of the quadratic term $\|X\|^2_{L^p}$. As a result, it is a polynomial Lipschitz function.
    We investigate the Jacobian matrix of the multi-head dot-product attention layer in the paper\citep{lipschitz21Kim}. In the Jacobian matrix, the element is $J_{ij} =  X^T P^{(i)}[E_{ji}XA^T+XA\delta_{ij}]+P_{ij}I$. Although it can be extremely large for large $\|X\|_{L^p}$, the entry is bounded by a polynomial $\|X\|_{L^p}$. As a result, it is a polynomial Lipschitz function.
    \item Transformer Layer: It is a combination of matrix multiplication, residual blocks, MLP blocks, Multi-head Attention, and activation functions. As a result, its polynomial Lipschitzable depends on the polynomial Lipschitz continuity of the selected activation functions.
    
\end{itemize}

\end{proof}

\subsection{Proof of Lemma \ref{lemma_3.6}}\label{Appendix_lemma_3.6_proof}

Recall the lemma 
\begin{lemma}
    Denote the matrix inversion function as $\text{Inv}(x)$. Suppose $A_0 \in \mathbb{R}^{n\times n}$ is a singular matrix with rank $n-1$, $B(A_0, \delta)$ is a ball centered at $A_0$ with radius $\delta$ in $\mathbb{R}^{n\times n}$. Denote $S_B$ as the set of all singular matrices in $B(A_0, \delta)$. Then, we can find a $\delta$ satisfies that for any matrix $A$ in $B(A_0, \delta) \setminus  S_B$, $\|\text{Inv}(A)\|_L>\frac{C_{A_0}}{\|A-A_0\|_L}$, where $C_{A_0}$ is a constant.
\end{lemma}

\begin{proof}
We analyse the matrix inversion by the formula
\begin{equation}
    A^{-1} = \frac{1}{det(A)}Adj(A)
\end{equation}
where $Adj(A)$ represents the adjugate matrix, and the $(i,j)-th$ element of $Adj(A)$ is $(-1)^{i+j}$ times the determinant of the $(n-1) \times (n-1)$ matrix that results from deleting row $j$ and column $i$ of $A$.

Because $det(A)=\sum_{\substack{\sigma \in S_n}} (sgn(\sigma)\prod _{i=1}^{n}a_{i,\sigma (i)}))$, the determinant of $A$ is the linear combination of the multiplication of specific elements. Hence, for a small enough positive number $\delta^{A_0}$, there exists a $C_{A_0}^{det}$ satisfies that for any $A \in B(A_0, \delta^{A_0})$,  $|det(A)-det(A_0)| < C_{A_0}^{det}\|A-A_0\|_{L_{\infty}}$. 

Because $\|A-A_0\|_{L_{\infty}} \leq \|A-A_0\|_{L_2}$, and $\|A-A_0\|_{L_{\infty}} \leq \|A-A_0\|_{L_1}$, we can simply state that $|det(A)-det(A_0)| < C_{A_0}^{det}\|A-A_0\|_{L}$.
Here, $C_{A_0}^{det}$ is a constant calculated from $A_0$. Because $A_0$ is a singular matrix, we have $det(A_0) =0$. As a result, in $B(A_0, \delta^{A_0})$,  $|det(A)|<C_{A_0}^{det}\|A-A_0\|_L$.

Under the $L-$ norm (either $L_1$, $L_2$, $L_\infty$), because the element of $Adj(A)$ is the determinant of the sub-matrix, it is easy to prove for any small positive number $\epsilon$, we can find a $\delta^{Adj}(\epsilon)$ satisfies that for any $A \in B(A_0, \delta^{Adj}(\epsilon))$,  $\|Adj(A)-Adj(A_0)\|_L<\epsilon$.

It is assumed that $A_0$ is rank $n-1$, which means that $Adj(A_0)$ is not zero-matrix. Hence, $\|Adj(A_0)\|_L>0$.

Set $\delta = min(\delta^{Adj}(\|Adj(A_0)\|_L/2), \delta^{A_0})$ and $C_{A_0}=\frac{\|Adj(A_0)\|_L}{2C_{A_0}^{det}}$. Then, for any  $A \in B(A_0, \delta) \setminus  S_B$,
\begin{equation}
    \begin{split}
        \|A^{-1}\|_L &= \frac{1}{|det(A)|}\|Adj(A)\|_L\\
        &\geq \frac{1}{|det(A)|}( \|Adj(A_0)\|_L-\|Adj(A)-Adj(A_0)\|_L)\\
        &> \frac{1}{2|det(A)|} \|Adj(A_0)\|_L\\
        &> \frac{1}{2C_{A_0}^{det}\|A-A_0\|_L} \|Adj(A_0)\|_L\\
        &=\frac{C_{A_0}}{\|A-A_0\|_L}\\
    \end{split}
\end{equation}
\end{proof}

\subsection{Proof for evaluation function $\mathbb{E}_M(\frac{\|\text{Inv}(x) - F(x)\|^K_L}{\|x\|^{K'}_L})$}

\begin{theorem}\label{main_theorem_math_2}
    Suppose the data is sampled from a dataset $M$ in $\mathbb{R}^{n\times n}$, with no singular matrix contained, and $B(\vec{a}, c)$ is contained in the set, which is a ball area and $\vec{a}$ in a data point, $c$ is a sufficiently large number. Then, under either $L_1$ norm, $L_2$ norm, or $L_{\infty}$ norm metric, for any polynomial Lipschitz continuous function $F(x)$, $\mathbb{E}_{M}(\frac{\|\text{Inv}(x) - F(x)\|^K_L}{\|x\|^{K'}_L})=+\infty$ under Lebesgue measurement if $K\geq  n^2$.
\end{theorem}

\begin{proof}
Suppose the measure of set $M$ is $m(M)$. We ignore all singular matrices in the set $M$ because the measure of the singular matrix set is 0 and denote it as $S_M$. From the definition of expectation over Lebesgue measurement, 
\begin{equation}
    \mathbb{E}_M(\frac{\|\text{Inv}(x) - F(x)\|^K_L}{\|x\|^{K'}_L}) = \int_{M\setminus S_M} \frac{\|\text{Inv}(x) - F(x)\|^K_L}{\|x\|^{K'}_L}\frac{1}{m(M)}dm
\end{equation}

Because $B(\vec{a}, c)$ is contained in the set, which is a ball area and $\vec{a}$ in a data point, $c$ is a sufficiently large number, there should be a singular matrix $A_0$ with rank $n-1$ with the ball $B(A_0, \epsilon_0)\setminus S_M$ contained in the dataset.

Obviously, $\mathbb{E}_M(\frac{\|\text{Inv}(x) - F(x)\|^K_L}{\|x\|^{K'}_L}) \geq \int_{B(A_0, \epsilon_0)\setminus S_M} \frac{\|\text{Inv}(x) - F(x)\|^K_L}{\|x\|^{K'}_L}\frac{1}{m(M)}dm$, we only consider the integral over set $B(A_0, \epsilon_0)$.

Because $F(x)$ is polynomial Lipschitz in $B(A_0, \epsilon_0)$, there must be a maximum value of $\|F(x)\|_L$, denote it as $C_F$.

In Lemma \ref{lemma_3.6}, we find that, for sufficiently small $\epsilon$, in the ball $B(A_0, \epsilon)$, $\|Adj(A) - Adj(A_0)\|_L< \|Adj(A_0)\|_L/2$ and  $|det(A)| = |det(A)-det(A_0)| < C_{A_0}^{det}\|A-A_0\|_{L_{\infty}}$. Because $\|A_0\|_L>0$ (otherwise $A_0$ is a zero matrix, rank is 0 not $n-1$), we can find a $\epsilon$ small enough to make all matrices in $B(A_0, \epsilon)$ satisfy $\|x\|_L \leq 2\|A_0\|_L$. Hence, we can set $\epsilon$ small enough to satisfy $\epsilon < \epsilon_0$ and $\|x-A_0\|_L<\frac{\|Adj(A_0)\|_L}{4C_{A_0}^{det}C_F}$, together with $\|x\|_L \leq 2\|A_0\|_L$. Then we have
\begin{equation}
\small
\begin{split}
&\|\frac{1}{det(A)}Adj(A) - F(x)\|_L \\
&\geq \frac{1}{det(A)}\|Adj(A)\|_L - \|F(x)\|_L \\
&\geq \frac{1}{det(A)}( \|Adj(A_0)\|_L-\|Adj(A)-Adj(A_0)\|_L) - \|F(x)\|_L \\
&> \frac{1}{2det(A)}\|Adj(A_0)\|_L - C_F \\
&> \frac{1}{2C_{A_0}^{det}\|A-A_0\|_{L}}\|Adj(A_0)\|_L - C_F \\
&= \frac{1}{4C_{A_0}^{det}\|A-A_0\|_{L}}\|Adj(A_0)\|_L>0
\end{split}
\end{equation}

Hence, 
\begin{equation}
\begin{split}
    &\int_{B(A_0, \epsilon) \setminus S_M} \frac{\|\text{Inv}(x) - F(x)\|^K_L}{\|x\|^{K'}_L}\frac{1}{m(M)}dm \\
    &= \int_{B(A_0, \epsilon) \setminus S_M} \|\frac{1}{det(A)}Adj(A) - F(x)\|^K_L \frac{1}{\|x\|^{K'}_L}\frac{1}{m(M)}dm \\
    &> \int_{B(A_0, \epsilon) \setminus S_M} (\frac{1}{4C_{A_0}^{det}\|A-A_0\|_{L}}\|Adj(A_0)\|_L)^K \frac{1}{2^{K'}\|A_0\|^{K'}_L}\frac{1}{m(M)}dm\\
    &= \frac{1}{m(M)}(\frac{\|Adj(A_0)\|_L}{4C_{A_0}^{det}})^K \frac{1}{2^{K'}\|A_0\|^{K'}_L} \int_{B(A_0, \epsilon) \setminus S_M} (\frac{1}{\|A-A_0\|_{L}})^K dm\\
    &\geq \frac{Const}{m(M)}(\frac{\|Adj(A_0)\|_L}{4C_{A_0}^{det}})^K \frac{1}{2^{K'}\|A_0\|^{K'}_L} \int_{0}^{\epsilon} r^{n^2-1-K}dr, \\
\end{split}
\end{equation}

where $Const$ represents a real number calculated from $n,C_{A_0},m(M)$.

Because $K\geq n^2$, obviously the last formula larger than any real number when $\epsilon \rightarrow \infty$.

Hence, we have

\begin{equation}
\begin{split}
    &\mathbb{E}_M(\frac{\|\text{Inv}(x) - F(x)\|^K_L}{\|x\|^{K'}_L}) \\
    &= \int_{M\setminus S_M} \frac{\|\text{Inv}(x) - F(x)\|^K_L}{\|x\|^{K'}_L}\frac{1}{m(M)}dm \\
    &\geq \int_{B(A_0, \epsilon) \setminus S_M} \frac{\|\text{Inv}(x) - F(x)\|^K_L}{\|x\|^{K'}_L}\frac{1}{m(M)}dm\\
    &>E.
\end{split}
\end{equation}
for any real number E.

As a result, $\mathbb{E}_M(\frac{\|\text{Inv}(x) - F(x)\|^K_L}{\|x\|^{K'}_L}) = +\infty$.

\end{proof}

\section{Hyperparameters in Experiment \ref{experiment_result}} \label{hyperparameter_appendix}

\begin{itemize}
    \item Adam optimizer: learning rate 5e-5, weight decay coefficient 1e-7.
    \item Warm restart: CosineAnnealingWarmRestarts function in PyTorch, $T\_0=3$,$T\_mult=2$,$eta\_min=1e-6$.
    \item Loss function: MSE Loss.
    \item $2\times 2 (1-st)$ dataset, 2 FC with ReLU: 
    \begin{itemize}
        \item First layer input features 4, output features 32. Second layer input features 32, output features 4.
        \item Batch size 128, training data contains 1,000,000 matrices, train 20 epochs. The test set contains 10,000 matrices.
        \item Trained for less than 1 hour.
    \end{itemize} 
    \item $2\times 2 (1-st)$ dataset, 3 FC with ReLU: 
    \begin{itemize}
        \item First layer input features 4, output features 32. The second layer input features 32, output features 32. The third layer input features 32, output features 4. 
        \item Batch size 128, training data contains 1,000,000 matrices, train 20 epochs. The test set contains 10,000 matrices.
        \item Trained for less than 1 hour.
    \end{itemize} 
    \item $2\times 2 (2-nd)$ dataset, 2 FC with ReLU: 
    \begin{itemize}
        \item First layer input features 4, output features 32. The second layer input features 32, output features 4.
        \item Batch size 128, training data contains 1,000,000 matrices, train 20 epochs. The test set contains 10,000 matrices.
        \item Trained for less than 1 hour.
    \end{itemize} 
    \item $2\times 2 (2-nd)$ dataset, 3 FC with ReLU: 
    \begin{itemize}
        \item First layer input features 4, output features 32. The second layer input features 32, output features 32. The third layer input features 32, output features 4. 
        \item Batch size 128, training data contains 1,000,000 matrices, train 20 epochs. The test set contains 10,000 matrices.
        \item Trained for less than 1 hour.
    \end{itemize} 
    \item $3\times 3 $ dataset, 2 FC with ReLU: 
    \begin{itemize}
        \item First layer input features 9, output features 72. The second layer input features 72, output features 9.
        \item Batch size 128, training data contains 100,000 matrices, train 200,000 steps. The test set contains 10,000 matrices.
        \item Trained for less than 1 hour.
    \end{itemize} 
    \item $3\times 3$ dataset, 3 FC with ReLU: 
    \begin{itemize}
        \item First layer input features 9, output features 72. The second layer input features 72, output features 72. The third layer input features 72, output features 9. 
        \item Batch size 128, training data contains 100,000 matrices, train 200,000 steps. The test set contains 10,000 matrices.
        \item Trained for less than 1 hour.
    \end{itemize} 
    \item $16\times 16 $ dataset, 2 FC with ReLU: 
    \begin{itemize}
        \item First layer input features 256, output features 2048. The second layer input features 2048, output features 256.
        \item Batch size 128, training data are generated during training, train 200,000 steps.
        \item Trained for less than 1 hour.
    \end{itemize} 
    \item $16\times 16$ dataset, 3 FC with ReLU: 
    \begin{itemize}
        \item First layer input features 256, output features 2048. The second layer input features 2048, output features 2048. The third layer input features 2048, output features 256. 
        \item Batch size 128, training data are generated during training, train 200,000 steps.
        \item Trained for less than 1 hour.
    \end{itemize} 
\end{itemize}

\section{Hyperparameters in Experiment \ref{experiment_theoretical}}\label{hyperparameter_appendix_theoretical}

\begin{itemize}
    \item Adam optimizer: learning rate 5e-5, weight decay coefficient 1e-7.
    \item Warm restart: CosineAnnealingWarmRestarts function in PyTorch, $T\_0=3$,$T\_mult=2$,$eta\_min=1e-6$.
    \item Loss function: MSE Loss.
    \item First layer input features 4, output features 8. The second layer is input feature 8, output features 4.
    \item Batch size 128, training data contains 1,000,000 matrices, train 20 epochs.
    \item Trained for less than 1 hour.
\end{itemize}

\section{Inference time}\label{inference_time_section}

In this section, we list the inference time for different MLP models and different datasets in Table \ref{inference_time_table}.

We try two MLP models: MLP model with 2 fully-connect layers and MLP model with 3 fully-connect layers, three datasets: $2\times 2$ matrix, $3\times 3$ matrix, $16\times 16$ matrix, and list the inference time for four different inference methods: inference by GPU with batchsize 1, inference by GPU with batchsize 100, inference by GPU with batchsize 10000, inference by CPU. Then, we compare these methods with the time of exact computing matrix inversion.

\begin{table}[t]
  \centering
  \small
  \caption{The inference time for 10000 samples. The experiments were run on a 3080 Laptop GPU and Intel i7-10870H CPU. bs represents the batch size}
  \label{inference_time_table}
  \vskip.05in
  \renewcommand{\arraystretch}{1.3}
  \begin{tabular}{c|cccc|c}
    \toprule
    Model &  bs:1 (GPU) & bs:100 (GPU) & bs:10000 (GPU) & bs:1 (CPU) & exact computing (CPU)\\
    \midrule
    2-FC MLP on $2\times 2$ matrix & 4.491s & 0.214s & 0.125s & 1.703s & 0.075s\\
    3-FC MLP on $2\times 2$ matrix & 8.177s & 0.356s & 0.162s & 2.032s & \\
    \midrule
    2-FC MLP on $3\times 3$ matrix & 5.501s & 0.257s & 0.162s & 1.700s & 0.078s\\
    3-FC MLP on $3\times 3$ matrix & 6.718s & 0.280s & 0.159s & 2.057s & \\
    \midrule
    2-FC MLP on $16\times 16$ matrix & 4.709s & 0.623s & 0.467s & 5.232s & 1.196s\\
    3-FC MLP on $16\times 16$ matrix & 6.623s & 0.878s & 0.459s & 25.887s & \\
    \bottomrule
  \end{tabular}
\end{table}

\section{Full analysis of model trained in Experiment \ref{experiment_theoretical}}\label{appendix_full_analysis}

In this section, all values are rounded to 5 significant figures. All the double-point precision values, together with codes/pre-trained models can be downloaded from our codebase in the supplement file. 

\subsection{Parameters of the trained model}
In the trained-well model, there are two fully connected layers. Hence, the formula of the neural network can be written as $y=RELU(xW_1^T+b_1)W_2^T+b_2$. We list all the values of the parameters of the trained model below:
\begin{itemize}
    \item The weight matrix $W_1$ of 1-st fully connected layer: \\
         $[ 0, 0, 0, 0]$\\
         $[ 9.5628e-02,  2.9369e-01,  5.3154e-02, -4.4647e-01]$\\
         $[ 0, 0,  0,  0]$\\
         $[ 1.5360e+00, -1.2078e+00, -7.5969e-01,  6.3958e-01]$\\
         $[-3.6153e-01, -4.7281e-02,  4.6443e-01, -3.2432e-01]$\\
         $[ 3.5525e-01, -1.8960e-01, -6.1029e-01,  1.3732e-01]$\\
         $[-1.2802e+00,  1.0202e+00,  6.6826e-01, -5.6571e-01]$\\
         $[-6.5787e-01,  4.8745e-01,  2.9334e-01, -2.2843e-01]$\\
    \item The bias $b_1$ of 1-st fully connected layer: \\
       $[0,  7.9882e-01, 0, -2.4201e-01,  1.0407e+00,
          1.2244e+00,  1.9913e-01,  9.9282e-02]$
    \item The weight matrix $W_2$ of 2-nd fully connected layer: \\
         $[0, -7.5081e-02,  0, -1.1650e+00,  5.8385e-01,
          -5.0147e-01,  1.1485e+00,  6.6387e-01]$\\
         $[ 0, -8.2578e-01, 0,  9.2669e-01,  1.7431e-01,
           4.8691e-01, -1.0013e+00, -3.4407e-01]$\\
         $[0, -4.9190e-01,  0,  5.3894e-01, -8.0408e-01,
           1.0796e+00, -6.3829e-01, -1.4140e-01]$\\
         $[ 0,  1.0662e+00,  0, -4.5637e-01,  3.7282e-01,
          -6.5138e-01,  5.4483e-01,  1.0231e-01]$\\
    \item The bias $b_2$ of 2-nd fully connected layer: \\
        $[-0.23847,  0.12373,  0.055120, -0.56574]$
\end{itemize}

We define 
\begin{equation}
\begin{split}
&h_1(a,b,c,d) = 0.095628 *a+ 0.29369*b+ 0.053154*c-0.44647*d + 0.79882\\
&h_3(a,b,c,d) =  1.5360 *a -1.2078*b -0.75969*c+  0.63958*d-0.24201\\
&h_4(a,b,c,d) = -0.36153*a -0.047281*b+  0.46443*c -0.32432*d+1.0407\\
&h_5(a,b,c,d) =  0.35525e *a -0.18960*b -0.61029*c+  0.13732*d+1.2244\\
&h_6(a,b,c,d) = -1.2802 *a+  1.0202*b+  0.66826*c -0.56571*d+0.19913\\
&h_7(a,b,c,d) = -0.65787e *a+  0.48745*b+  0.29334*c -0.22843*d+0.099282\\
\end{split}
\end{equation}
  
\subsection{Analysis of computing $a_{ij}$ in pre-trained neural network}
Consider the element $a_{ij}$, in this network, it can be represented as $a_{ij}=\sum_{k=0}^7 w_{k,ij}\text{ReLU}(h_k(a,b,c,d))$. For example, $a_{11}$ has a form of

\begin{equation}\label{a_11_neural_network}
\begin{split}
    a_{11} =& -(7.5081e-02)*\text{ReLU}(h_1)\\
&-(1.1650e+00)*\text{ReLU}(h_3)\\
&+(5.8385e-01)*\text{ReLU}(h_4)\\
&-(5.0147e-01)*\text{ReLU}(h_5)\\
&+(1.1485e+00)*\text{ReLU}(h_6)\\
&+(6.6387e-01)*\text{ReLU}(h_7)\\
&-0.23847\\
\end{split}
\end{equation}

We use two methods, random sample by experiments and linear programming, to show that the pre-trained neural network has learned the linear approximation of matrix inversion.
\subsubsection{Experiments: Random sample}
Recall the dataset area is $\prod_{i=1,j=1}^{2}[A_{0,i,j} - 0.01,A_{0,i,j} + 0.01] \in \mathbb{R}^{2\times 2}$, and $A_0=\begin{pmatrix}
2 & 2 \\
2 & 3
\end{pmatrix}$.

We randomly sampled 1M data points in the dataset area and found $55.7$ percent of data located in the area $\{h_i>0|i \in \{1,4,5,6,7\}\}\cap \{h_i<0|i \in \{3\}\}$, and $41.7$ percent of data located in the area $\{h_i>0|i \in \{1,3,4,5\}\}\cap \{h_i<0|i \in \{6,7\}\}$. These two cases include most of the data ($97.4\%$) in the area.

In the first case, we eliminate the ReLU function in formula \ref{a_11_neural_network} and get
\begin{equation}
\begin{split}
     a_{11} &= -2.3034*a + 1.5408*b + 1.5354*c - 1.0260*d - 0.0102 \\
     a_{12} &=  1.5392*a - 1.5324*b - 1.0302*c + 1.0241*d + 0.0081 \\
     a_{21} &=  1.5373*a - 1.0313*b - 1.5265*c + 1.0220*d + 0.0060\\
     a_{22} &= -1.0290*a + 1.0248*b + 1.0215*c - 1.0180*d - 0.0049\\
\end{split}
\end{equation}

In the second case, we eliminate the ReLU too:

\begin{equation}
\begin{split}
    a_{11} &= -2.1860*a + 1.4526*b + 1.4583*c - 0.9698*d - 0.0229\\
    a_{12} &=  1.4544*a - 1.4624*b - 0.9641*c + 0.9717*d + 0.0174\\
    a_{21} &=  1.4550*a - 0.9621*b - 1.4679*c + 0.9733*d + 0.0167\\
    a_{22} &= -0.9652*a + 0.9702*b + 0.9741*c - 0.9783*d - 0.0131\\
\end{split}
\end{equation}

Compare with the linear approximation 
\begin{equation}
    \begin{split}
        &a_{11}\approx - 2.25a+ 1.5b+ 1.5c- d\\
        &a_{12}\approx 1.5a- 1.5b-c+d\\
        &a_{21}\approx 1.5a- b-1.5c+d\\
        &a_{22}\approx  - a + b+ c-  d\\
    \end{split}
\end{equation}
, we can find that the distance between each coefficient in the neural network and each coefficient in the linear approximation is smaller than $0.06$.

\subsubsection{Linear Programming}

In this section, we consider that, in each of the area $\cap_k \{(a,b,c,d)|h_k(a,b,c,d)>0\ or\ <0\}$, how large is the distance between the output of the neural network and the linear approximation.

For $a_{ij}$, this problem can be stated as a linear programming problem:
\begin{equation}
    \begin{split}
        &Found\ (a,b,c,d)\ maximize\ |\sum_{k=0}^7 w_{k,ij}\text{ReLU}(h_k(a,b,c,d)) + bias_{ij} - Linear_{ij}(a,b,c,d)|\\
        &Subject\ to\ (a,b,c,d)\in [-c,c]^4\\
        &And\ \{h_k>0(or h_k<0)\}
    \end{split}
\end{equation},
where the $Linear_{ij}$ function represents
\begin{equation}
    \begin{split}
        &Linear_{11}(a,b,c,d)= - 2.25a+ 1.5b+ 1.5c- d\\
        &Linear_{12}(a,b,c,d)= 1.5a- 1.5b-c+d\\
        &Linear_{21}(a,b,c,d)= 1.5a- b-1.5c+d\\
        &Linear_{22}(a,b,c,d)=  - a + b+ c-  d\\
    \end{split}
\end{equation}

For each area, we test linear programming on it. However, many areas, like all the $h_k<0$, has no overlap between $(a,b,c,d)\in [-c,c]^4$, which is a null-set. As a result, we only test the area that contains data. The results for each $a_{ij}$ are shown in Table \ref{a_11_maximum_abs_linear}, \ref{a_12_maximum_abs_linear}, \ref{a_21_maximum_abs_linear}, \ref{a_22_maximum_abs_linear} separately:
\begin{table}[ht]
  \centering
  \small
  \caption{Maximum absolute value between the linear approximation of $a_{11}$ and pre-trained neural network in sets which are not null-set.}
  \label{a_11_maximum_abs_linear}
  \vskip.05in
      \renewcommand{\arraystretch}{1.3}
  \begin{tabular}{c|c|c}
    \toprule
    Data proportion & Area & Maximum absolute value \\
    \midrule
    $<0.01\%$ &$h_1>0, h_3>0, h_4>0, h_5>0, h_6>0, h_7>0$ & $0.00024328$ \\
    $0.26\%$ &$h_1>0, h_3>0, h_4>0, h_5>0, h_6>0, h_7<0$ & $0.00046950$ \\
    $0.88\%$ &$h_1>0, h_3>0, h_4>0, h_5>0, h_6<0, h_7>0$ & $0.00054643$ \\
    $41.72\%$ &$h_1>0, h_3>0, h_4>0, h_5>0, h_6<0, h_7<0$ & $0.0015488$\\
    $0.06\%$ &$h_1>0, h_3>0, h_4<0, h_5>0, h_6<0, h_7<0$ & $0.0019683$ \\
    $55.67\%$ &$h_1>0, h_3<0, h_4>0, h_5>0, h_6>0, h_7>0$ & $0.0014152$ \\
    $0.92\%$ &$h_1>0, h_3<0, h_4>0, h_5>0, h_6>0, h_7<0$ & $0.00046950$ \\
    $0.44\%$ &$h_1>0, h_3<0, h_4>0, h_5>0, h_6<0, h_7>0$ & $0.00054643$ \\
    $0.02\%$ &$h_1>0, h_3<0, h_4>0, h_5>0, h_6<0, h_7<0$ & $0.00027471$ \\
    $<0.01\%$ &$h_1>0, h_3<0, h_4>0, h_5<0, h_6>0, h_7>0$ & $0.0014152$ \\
    $<0.01\%$ &$h_1<0, h_3>0, h_4>0, h_5>0, h_6<0, h_7>0$ & $0.00037102$ \\
    $0.02\%$ &$h_1<0, h_3>0, h_4>0, h_5>0, h_6<0, h_7<0$ & $0.00092748$ \\
    $<0.01\%$ &$h_1<0, h_3<0, h_4>0, h_5>0, h_6>0, h_7>0$ & $0.00020061$ \\
    $<0.01\%$ &$h_1<0, h_3<0, h_4>0, h_5>0, h_6<0, h_7>0$ & $0.00037102$\\
    \bottomrule
  \end{tabular}
\end{table}

\begin{table}[ht]
  \centering
  \small
  \caption{Maximum absolute value between the linear approximation of $a_{12}$ and pre-trained neural network in sets which are not null-set.}
  \label{a_12_maximum_abs_linear}
  \vskip.05in
      \renewcommand{\arraystretch}{1.3}
  \begin{tabular}{c|c|c}
    \toprule
    Data proportion & Area & Maximum absolute value \\
    \midrule
    $<0.01\%$ &$h_1>0, h_3>0, h_4>0, h_5>0, h_6>0, h_7>0$ & $0.00018972$ \\
    $0.26\%$ &$h_1>0, h_3>0, h_4>0, h_5>0, h_6>0, h_7<0$ & $0.00031370$ \\
    $0.88\%$ &$h_1>0, h_3>0, h_4>0, h_5>0, h_6<0, h_7>0$ & $0.00041427$ \\
    $41.72\%$ &$h_1>0, h_3>0, h_4>0, h_5>0, h_6<0, h_7<0$ & $0.0012519$\\
    $0.06\%$ &$h_1>0, h_3>0, h_4<0, h_5>0, h_6<0, h_7<0$ & $0.0012519$ \\
    $55.67\%$ &$h_1>0, h_3<0, h_4>0, h_5>0, h_6>0, h_7>0$ & $0.0011440$ \\
    $0.92\%$ &$h_1>0, h_3<0, h_4>0, h_5>0, h_6>0, h_7<0$ & $0.00031370$ \\
    $0.44\%$ &$h_1>0, h_3<0, h_4>0, h_5>0, h_6<0, h_7>0$ & $0.00041427$ \\
    $0.02\%$ &$h_1>0, h_3<0, h_4>0, h_5>0, h_6<0, h_7<0$ & $0.00021671$ \\
    $<0.01\%$ &$h_1>0, h_3<0, h_4>0, h_5<0, h_6>0, h_7>0$ & $0.0011440$ \\
    $<0.01\%$ &$h_1<0, h_3>0, h_4>0, h_5>0, h_6<0, h_7>0$ & $0.00049615$ \\
    $0.02\%$ &$h_1<0, h_3>0, h_4>0, h_5>0, h_6<0, h_7<0$ & $0.0011328$ \\
    $<0.01\%$ &$h_1<0, h_3<0, h_4>0, h_5>0, h_6>0, h_7>0$ & $0.00022353$ \\
    $<0.01\%$ &$h_1<0, h_3<0, h_4>0, h_5>0, h_6<0, h_7>0$ & $0.00049615$\\
    \bottomrule
  \end{tabular}
\end{table}

\begin{table}[ht]
  \centering
  \small
  \caption{Maximum absolute value between the linear approximation of $a_{21}$ and pre-trained neural network in sets which are not null-set.}
  \label{a_21_maximum_abs_linear}
  \vskip.05in
      \renewcommand{\arraystretch}{1.3}
  \begin{tabular}{c|c|c}
    \toprule
    Data propotion & Area & Maximum absolute value \\
    \midrule
    $<0.01\%$ &$h_1>0, h_3>0, h_4>0, h_5>0, h_6>0, h_7>0$ & $0.00014555$ \\
    $0.26\%$ &$h_1>0, h_3>0, h_4>0, h_5>0, h_6>0, h_7<0$ & $0.00024849$ \\
    $0.88\%$ &$h_1>0, h_3>0, h_4>0, h_5>0, h_6<0, h_7>0$ & $0.00020619$ \\
    $41.72\%$ &$h_1>0, h_3>0, h_4>0, h_5>0, h_6<0, h_7<0$ & $0.0012065$\\
    $0.06\%$ &$h_1>0, h_3>0, h_4<0, h_5>0, h_6<0, h_7<0$ & $ 0.0017523$ \\
    $55.67\%$ &$h_1>0, h_3<0, h_4>0, h_5>0, h_6>0, h_7>0$ & $0.0010813$ \\
    $0.92\%$ &$h_1>0, h_3<0, h_4>0, h_5>0, h_6>0, h_7<0$ & $0.00016381$ \\
    $0.44\%$ &$h_1>0, h_3<0, h_4>0, h_5>0, h_6<0, h_7>0$ & $0.00024714$ \\
    $0.02\%$ &$h_1>0, h_3<0, h_4>0, h_5>0, h_6<0, h_7<0$ & $0.00016381$ \\
    $<0.01\%$ &$h_1>0, h_3<0, h_4>0, h_5<0, h_6>0, h_7>0$ & $0.0010813$ \\
    $<0.01\%$ &$h_1<0, h_3>0, h_4>0, h_5>0, h_6<0, h_7>0$ & $0.00029226$ \\
    $0.02\%$ &$h_1<0, h_3>0, h_4>0, h_5>0, h_6<0, h_7<0$ & $0.00079305$ \\
    $<0.01\%$ &$h_1<0, h_3<0, h_4>0, h_5>0, h_6>0, h_7>0$ & $0.00019002$ \\
    $<0.01\%$ &$h_1<0, h_3<0, h_4>0, h_5>0, h_6<0, h_7>0$ & $0.00025328$\\
    \bottomrule
  \end{tabular}
\end{table}

\begin{table}[ht]
  \centering
  \small
  \caption{Maximum absolute value between the linear approximation of $a_{22}$ and pre-trained neural network in sets which are not null-set.}
  \label{a_22_maximum_abs_linear}
  \vskip.05in
      \renewcommand{\arraystretch}{1.3}
  \begin{tabular}{c|c|c}
    \toprule
    Data proportion & Area & Maximum absolute value \\
    \midrule
    $<0.01\%$ &$h_1>0, h_3>0, h_4>0, h_5>0, h_6>0, h_7>0$ & $0.00011774$ \\
    $0.26\%$ &$h_1>0, h_3>0, h_4>0, h_5>0, h_6>0, h_7<0$ & $0.00021507$ \\
    $0.88\%$ &$h_1>0, h_3>0, h_4>0, h_5>0, h_6<0, h_7>0$ & $0.00017554$ \\
    $41.72\%$ &$h_1>0, h_3>0, h_4>0, h_5>0, h_6<0, h_7<0$ & $0.00095610$\\
    $0.06\%$ &$h_1>0, h_3>0, h_4<0, h_5>0, h_6<0, h_7<0$ & $ 0.0012230$ \\
    $55.67\%$ &$h_1>0, h_3<0, h_4>0, h_5>0, h_6>0, h_7>0$ & $0.00085934$ \\
    $0.92\%$ &$h_1>0, h_3<0, h_4>0, h_5>0, h_6>0, h_7<0$ & $0.00013245$ \\
    $0.44\%$ &$h_1>0, h_3<0, h_4>0, h_5>0, h_6<0, h_7>0$ & $0.00021283$ \\
    $0.02\%$ &$h_1>0, h_3<0, h_4>0, h_5>0, h_6<0, h_7<0$ & $0.00013245$ \\
    $<0.01\%$ &$h_1>0, h_3<0, h_4>0, h_5<0, h_6>0, h_7>0$ & $0.00085934$ \\
    $<0.01\%$ &$h_1<0, h_3>0, h_4>0, h_5>0, h_6<0, h_7>0$ & $0.00054635$ \\
    $0.02\%$ &$h_1<0, h_3>0, h_4>0, h_5>0, h_6<0, h_7<0$ & $0.0012563$ \\
    $<0.01\%$ &$h_1<0, h_3<0, h_4>0, h_5>0, h_6>0, h_7>0$ & $0.00016273$ \\
    $<0.01\%$ &$h_1<0, h_3<0, h_4>0, h_5>0, h_6<0, h_7>0$ & $0.00042429$\\
    \bottomrule
  \end{tabular}
\end{table}

\end{document}